\documentclass[12pt]{elsarticle}
\usepackage{amsmath}
\usepackage{amssymb}
\usepackage{mathrsfs}
\usepackage{amsthm}
\usepackage{bm}
\usepackage{booktabs}
\usepackage{color}
\usepackage{caption}
\usepackage{dashrule}
\usepackage{epstopdf}
\usepackage[right=1.5cm,left=1.5cm,top=1.0cm,bottom=1.5cm]{geometry}
\usepackage{graphicx}
\usepackage{lineno}
\usepackage{multirow}
\usepackage{newtxtext}    
\usepackage{newtxmath}    
\usepackage{microtype}
\usepackage{mathtools}
\usepackage{natbib}
\usepackage{setspace}
\usepackage{sectsty} 
\usepackage{threeparttable}
\usepackage{subfigure} 
\usepackage{tabularx} 
\usepackage{titlesec} 
\usepackage{enumitem} 
\usepackage{algorithm}
\usepackage{algorithmicx}
\usepackage{algpseudocode} 
\usepackage{amsmath, amssymb, amsthm, mathtools}
\usepackage{hyperref} 
\usepackage[capitalise]{cleveref} 

\makeatletter
\renewcommand\fs@ruled{%
	\def\@fs@cfont{\bfseries}%
	\let\@fs@capt\floatc@ruled
	\def\@fs@pre{\hrule height.8pt depth0pt \kern0.4em}
	\def\@fs@post{\kern0.4em \hrule height.8pt depth0pt \kern0.4em}
	\def\@fs@mid{\kern0.4em \hrule height.8pt depth0pt \kern0.6em}
	\let\@fs@iftopcapt\iftrue
}
\makeatother

\sectionfont{\fontsize{14}{16}\selectfont\bfseries} 
\subsectionfont{\fontsize{12}{14}\selectfont\bfseries} 
\titlespacing*{\section}{0pt}{12pt plus 4pt minus 2pt}{6pt plus 2pt minus 1pt}
\titlespacing*{\subsection}{0pt}{8pt plus 2pt minus 1pt}{4pt plus 1pt minus 1pt}

\hypersetup{colorlinks = true, linkcolor=blue, citecolor=blue, hypertexnames=false} 
\newtheorem{thm}{Theorem}
\newtheorem{lem}{Lemma}    
\newtheorem{pro}{Proposition} 
\newtheorem{rmk}{Remark} 
\newtheorem{definition}{Definition}
\newproof{pf}{Proof}
\DeclarePairedDelimiter{\abs}{\lvert}{\rvert}
\DeclarePairedDelimiter{\norm}{\lVert}{\rVert}

\modulolinenumbers[10]
\journal{Knowledge-Based Systems}
\biboptions{numbers} 
\begin{document}
	\captionsetup[figure]{labelfont={bf},labelformat={default},labelsep=period,name={Fig.}}
	\captionsetup[table]{labelfont={bf},labelformat={default},labelsep=period,name={Table}}
	\begin{spacing}{1.5}
		
\begin{frontmatter}
\title{\texorpdfstring{$\text{NS}^3\text{R}$}{NS3R}: A knowledge-driven nonconvex square-root regularization framework for joint group sparsity and structured smoothness}

\author[1]{Wanling Xie}\ead{xwl@cqu.edu.cn} 

\author{Hu Yang\textsuperscript{b*}} 
\cortext[correspondingauthor]{Corresponding author. \emph{Email: yh@cqu.edu.cn}}

\address[1]{College of Mathematics and Statistics, Hunan University of Technology and Business, Changsha, Hunan 410205, China} 
\address[2]{College of Mathematics and Statistics, Chongqing University, Chongqing, 401331, China}

\begin{abstract}
Structured sparse learning is a leading paradigm for high-dimensional data modeling, as it enables rigorous integration of domain-validated structural priors into the learning process. Among these priors, pre-defined grouped structure and piecewise smoothness are ubiquitous in real-world applications. However, most existing structured regularization methods are built on squared loss and convex penalties, rarely jointly model these two complementary structural priors, and suffer from two key limitations. Their optimal tuning parameters are sensitive to the unknown noise level $\sigma$ due to the squared loss. Meanwhile, convex penalties induce inherent estimation bias and severe over-shrinkage of large magnitudes. To address these issues, we propose a knowledge-driven Nonconvex Structured Smooth Adjustment for Square-Root Regularization (NS$^3$R), which unifies the $\sigma$-free square-root loss with a joint nonconvex $\ell_{2,q}$ penalty ($0 < q < 1$) to simultaneously encode group sparsity and structured smoothness priors. Theoretically, we establish rigorous non-asymptotic error bounds for the NS$^3$R estimator under a mild $(2,q)$-group restricted compatibility condition. Computationally, we develop an efficient double-screened IRLS with thresholding (DS-IRLS-th) algorithm, which adopts a nested double-screening strategy to reduce computational cost and improve numerical stability, with provable global convergence for the resulting nonconvex optimization problem. Extensive experiments on synthetic and real-world data demonstrate that NS$^3$R consistently outperforms state-of-the-art methods in finite-sample settings.
\end{abstract}

\end{frontmatter}

\section{Introduction}
\hspace{2em} High-dimensional sparse learning serves as a fundamental methodology in modern statistics and applied mathematics. The explosive growth of data dimensionality has created great challenges for statistical modeling and inference across a wide range of real-world scientific domains, including compressed sensing \citep{Donoho2006}, genome-wide association studies \citep{Tharmaratnam2016, Zeng2021}, signal recovering \citep{Hu2017, Xie2020, Xue2020}, multi-task learning \citep{Chang2024}, econometric forecasting \citep{Saputra2026} and so on. In these practical applications, the relationship between the observed data and the target signal of interest can be well approximated by the linear system
\[
y = X\beta^* + \varepsilon \tag{1}
\label{eq:opt1}
\]
In this system, $X \in \mathbb{R}^{n\times p}$ denotes the measurement matrix, $y \in \mathbb{R}^n$ corresponds to the noisy observation vector, $\beta^* \in \mathbb{R}^p$ is the unknown true signal vector we aim to recover, and $\varepsilon$ is the random noise. Throughout this work, we adopt the standard Gaussian noise assumption, where $\varepsilon$ has zero mean and unknown covariance $\sigma^2 I_n$. A core property of high-dimensional scenarios is that the dimension $p$ is far greater than the number of observations $n$ (i.e., $p \gg n$), which renders conventional model selection and estimation techniques ineffective.

Building on the sparse learning framework widely used for high-dimensional data, structured sparse learning has emerged as the standard approach to address the limitations of traditional model selection methods under the $p \gg n$ setting. It moves beyond the entry-wise sparsity induced by the Lasso-type procedures \citep{Tibshirani1996} by leveraging inherent structural priors of the sparse signal we aim to recover, the most common and practically valuable of which are pre-defined group sparsity and fused sparsity. For group sparsity, the most widely used method is the group Lasso \citep{Yuan2006}, which adopts a convex $\ell_{2,1}$ penalty to induce group-level sparsity by encouraging elements in the same group to be either all non-zero or all zero simultaneously. Despite its popularity, the convex penalty brings excessive shrinkage on large non-zero coefficients resulting in non-negligible estimation bias, and tends to produce overly dense solutions leading to sub-optimal recovery performance in practice. To address these limitations, a range of nonconvex group penalties have been proposed, including the widely used group SCAD \citep{Wang2007, Breheny2015, Lee2016}, group MCP \citep{Breheny2015, Lee2016}, and capped folded concave penalties \citep{Pan2021, Zhang2022}. Also, the nonconvex $\ell_{r,q}$ regularization has drawn extensive attention for its superior flexibility in modeling group sparsity \citep{Hu2017, Xue2020, Rakotomamonjy2011}, with its general form given by
\[
\|\beta\|_{r,q}:=
\begin{cases}
	\left(\sum_{j=1}^J \|\beta_{\mathcal{G}_j}\|_r^q\right)^{\frac{1}{q}}, & q > 0 \\
	\sum_{j=1}^J \|\beta_{\mathcal{G}_j}\|_r^0, & q = 0
\end{cases}
\]
where $\beta_{\mathcal{G}_j}$ denotes the $j$-th group of $\beta$ (see \autoref{sec:methodology} for more details about group structure). As shown in \citep{Hu2017, Xue2020, Rakotomamonjy2011}, the nonconvex $\ell_{r,q}$ penalties achieve a more favorable trade-off between sparsity induction and estimation bias reduction than convex alternatives, and empirical evidence from their simulation studies further shows that $\ell_{r,q}$ regularization with $r=2$ outperforms its general $\ell_{r,q}$ counterparts in most cases. 

Fused sparsity for ordered features is another widely used structural prior in high-dimensional statistical modeling. The foundational method for this prior is the fused Lasso \citep{Tibshirani2005}, which adds a total variation penalty to the standard Lasso for modeling piecewise constant signal patterns over ordered features. It can adaptively identify contiguous blocks of identical coefficients without pre-specified group information, and has been widely adopted in signal processing, genomic analysis, and image denoising. Despite its popularity, the fused Lasso has clear limitations that restrict its practical utility. Its core design is tailored to capture step-function patterns rather than smooth local variations, which not only incurs a tremendous computational cost as the dimension increases, but also renders it fundamentally unsuitable for signals with smoothly varying coefficients. To tackle these limitations, multiple extensions have been developed. The Smooth Lasso \citep{Hebiri2011} replaces the original $\ell_1$ fusion penalty with a squared $\ell_2$ term to better capture gradual smooth shifts in coefficients, whereas the Spline Lasso \citep{Guo2016} draws on cubic spline smoothing ideas to model more flexible coefficient trajectories. Subsequent works have further extended this framework in model development, theoretical analysis and applied extensions \citep{Ma2017, Wen2024}. Nevertheless, most existing structured regularization methods focus only on either group structures or ordered structures, and cannot simultaneously incorporate both, which are ubiquitous in real-world applications.

Nearly all the structured sparse learning methods discussed above are built on the squared loss function. A critical limitation of these penalized least-squares is the challenge of regularization parameter tuning, as the theoretically optimal penalty values are closely tied to the unknown noise level $\sigma$. Accurate estimation of $\sigma$ is notoriously difficult in high-dimensional regimes, introducing practical hurdles and a heavy tuning burden in practical use. To address this critical issue, the square-root regularization framework was originally proposed. The classic square-root Lasso \citep{Belloni2011} replaces the squared loss with a square-root loss function, enabling scale-independent parameter tuning as the optimal regularization parameter no longer depends on the unknown noise level. This $\sigma$-free tuning property has motivated a series of extensions to structured sparse settings. The group square-root Lasso \citep{Bunea2013} extends the framework to group-level feature selection; the square-root fused Lasso \citep{Jiang2021} adapts the framework to capture local constant signal patterns; the SSASR model \citep{Xie2020} combines the square-root loss with a convex $\ell_{2,1}$ penalty to jointly model group sparsity and piecewise smoothness. Recent studies have further advanced this line of work by refining statistical theory and developing more efficient numerical algorithms \citep{Zhao2021, Chu2022, Wang2023, Gai2025}. However, the vast majority of existing square-root regularized methods adopt convex penalty terms, leading to unavoidable estimation bias and sub-optimal recovery performance. To the best of our knowledge, no existing work has proposed a nonconvex square-root regularization framework that unifies group sparsity and structured smoothness priors. Meanwhile, this challenging nonconvex optimization problem lacks systematic theoretical guarantees and efficient convergent algorithms.

Motivated by the aforementioned limitations in existing literature, we develop a novel nonconvex structured smooth adjustment for square-root regularization (NS$^3$R) framework, which combines the square-root loss function with a joint nonconvex $\ell_{2,q}$ penalty for high-dimensional linear regression. It simultaneously encodes pre-defined group sparsity and structured smoothness priors in the sparse signal, while alleviating the common drawbacks of $\sigma$-dependent tuning and over-shrinkage from convex regularization. The main contributions of this work are summarized as follows.
\begin{itemize}
	\item Based on a joint $\ell_{2,q}$ penalty, we propose a unified nonconvex square-root regularization framework to fill a critical gap in structured sparse learning. This framework enables joint modeling of group sparsity and structured smoothness while retaining the inherent $\sigma$-free tuning property of square-root regularization. It further mitigates the inherent over-shrinkage and estimation bias of conventional convex regularizers, and thus achieves more reliable recovery of weak signals in high-dimensional regimes.
    \item For the NS$^3$R estimator, we derive non-asymptotic Oracle inequalities under one of the mildest regularity conditions in existing literature. These results provide reliable theoretical guarantees for its practical application, and ensure stable finite-sample performance in high-dimensional regimes.
	\item To solve the challenging nonconvex, nonsmooth, and non-Lipschitz optimization problem, we develop an efficient double-screened IRLS with thresholding (DS-IRLS-th) algorithm. The proposed algorithm not only comes with a rigorous proof of global convergence under mild initial value conditions, but also reduces computational cost and improves numerical stability via its nested double-screening design. Comprehensive numerical experiments on both synthetic and real-world data further validate its fast convergence and favorable performance in $p \gg n$ settings.
\end{itemize}

The rest of the paper is organized as follows. In \autoref{sec:methodology}, we present the nonconvex structured smooth adjustment for square-root regularization (NS$^3$R) model and analyze its key theoretical properties. In \autoref{sec:algorithm}, we propose a double-screened IRLS with thresholding (DS-IRLS-th) algorithm admitting provable global convergence. In \autoref{sec:numerical} and \autoref{sec:application}, we conduct simulation studies and an application to a microarray gene expression dataset to demonstrate the superior finite-sample performance of the proposed method compared with competing approaches. We conclude with final remarks in \autoref{sec:conclusion}. All technical proofs are deferred to the Appendix.

\section{Methodology}\label{sec:methodology}
\subsection{The \texorpdfstring{$\text{NS}^3\text{R}$}{NS3R} model}
\hspace{2em} In this paper, we focus on recovering a sparse signal $\beta \in \mathbb{R}^p$ with a pre-defined grouped structure, i.e.,
\[
\beta = \left( \underbrace{\beta_{1,1}, \dots, \beta_{1,T_1}}_{\beta_{\mathcal{G}_1}}, \underbrace{\beta_{2,1}, \dots, \beta_{2,T_2}}_{\beta_{\mathcal{G}_2}}, \dots, \underbrace{\beta_{J,1}, \dots, \beta_{J,T_J}}_{\beta_{\mathcal{G}_J}} \right)^T.
\]
Here $\beta_{\mathcal{G}_j}$ denotes the sub-vector of $\beta$ formed by entries indexed in $\mathcal{G}_j$, where $\mathcal{G}_j$ is an index subset of $\{1, \dots, p\}$. $J$ is the total number of groups, and $T_j = |\mathcal{G}_j|$ denotes the size of the $j$-th group.
For any group $\mathcal{G}_j$, $\beta_{\mathcal{G}_j} = 0$ means $\beta_{j,k} = 0$ for all $k \in \{1, \dots, T_j\}$, while $\beta_{\mathcal{G}_j} \neq 0$ implies that all entries in the $j$-th group are non-zero. We assume in this paper that there is no overlap between any two groups, i.e., $\mathcal{G}_i \cap \mathcal{G}_k = \emptyset$ for $i \neq k$.

In practice, sparse signals often exhibit a piecewise (or blockwise) smooth structure, where the true entries within each block typically exhibit inherent similarity. Thus, incorporating within-group smoothness constraints into model design is well justified. Meanwhile, conventional convex regularization methods suffer from critical drawbacks (e.g., inherent estimation bias and sub-optimal recovery performance) in terms of both theoretical guarantees and practical performance. To address these issues, we extend the group square-root regularization to a nonconvex fused framework. First, given $0 < q < 1$ and weights $\omega \in \mathbb{R}_+^J$ (i.e., $\omega_j \ge 0$ for all $1 \le j \le J$), for any vector $x \in \mathbb{R}^p$ with a pre-defined group structure, we denote the $\ell_{2,q}$ regularization term as
	\[
	\|x\|_{\omega-(2,q)} = \left( \sum_{j=1}^J \omega_j^q \|x_{\mathcal{G}_j}\|_2^q \right)^{1/q}
	\]
	
Then we propose a novel nonconvex structured smooth adjustment for square-root regularization ($\text{NS}^3\text{R}$) as follows
   \begin{equation}
	\hat{\beta} = \arg\min_{\beta \in \mathbb{R}^p} \frac{1}{\sqrt{n}} \| y - X\beta \|_2 + \frac{\lambda_1}{n} \|\beta\|_{\omega-(2,q)}^q + \frac{\lambda_2}{n} \|D\beta\|_{\omega-(2,q)}^q \tag{2}
	\label{eq:opt2}
   \end{equation}
where $\lambda_1$ and $\lambda_2$ are tuning parameters that control the penalization in a different manner. To simplify notation, we let the overall regularization strength be $\lambda = \lambda_1 + \lambda_2$ and define a regularization weight parameter $\alpha = \lambda_1/\lambda$ that balances the relative contribution of the two penalization terms. Among the two nonconvex regularizers, the first term induces group-level sparsity in the signals, while the second is a nonconvex Inner Group Total Variation (IGTV) regularizer that imposes structured smoothness within groups, which is defined as
	\begin{equation}
	\|D\beta\|_{\omega-(2,q)}^q = \sum_{j=1}^J \omega_j^q \|D_{\mathcal{G}_j} \beta_{\mathcal{G}_j}\|_2^q = \sum_{j=1}^J \omega_j^q \left[ \sum_{k=2}^{T_j} \left( \beta_{j,k} - \beta_{j,k-1} \right)^2 \right]^{q/2} \tag{3}
	\label{eq:opt3}
    \end{equation}
where $D$ is a block-diagonal matrix with the first-order difference matrices $D_{\mathcal{G}_j}$ for the $j$th group as each of its diagonal blocks, i.e., $D = diag(D_{\mathcal{G}_1}, D_{\mathcal{G}_2}, \dots, D_{\mathcal{G}_J}) \in \mathbb{R}^{(p-J) \times p}$ and
    \[
    D_{\mathcal{G}_j} = \begin{bmatrix}
    	1 & -1 & 0 & \cdots & 0 & 0 \\
    	0 & 1 & -1 & \cdots & 0 & 0 \\
    	\vdots & \vdots & \vdots & \ddots & \vdots & \vdots \\
    	0 & 0 & 0 & \cdots & 1 & -1
    \end{bmatrix}_{(T_j-1) \times T_j}
    \]
    
There are many reasonable choices for the weights $\omega \in \mathbb{R}_+^J$, and a widely adopted choice in the literature is $\omega_j = \sqrt{T_j}$. We therefore fix $\omega_j = \sqrt{T_j}$ throughout this paper. Then the problem in \eqref{eq:opt2} is equivalent to the following optimization problem
    \begin{equation}
    \hat{\beta} = \arg\min_{\beta \in \mathbb{R}^p} \frac{1}{\sqrt{n}} \|y - X\beta\|_2 + \frac{\lambda}{n} \left\{ \alpha \sum_{j=1}^J T_j^{q/2} \|\beta_{\mathcal{G}_j}\|_2^q + (1 - \alpha) \sum_{j=1}^J T_j^{q/2} \|D_{\mathcal{G}_j} \beta_{\mathcal{G}_j}\|_2^q \right\} \tag{4}
    \label{eq:opt4}
   \end{equation}	
Note that the objective function in \eqref{eq:opt4} uses the square-root residual sum of squared errors as its loss function, with a joint $\ell_{2,q}$ penalty that simultaneously encodes group sparsity and structured smoothness priors. In addition, problem \eqref{eq:opt4} serves as a general nonconvex extension of the SSASR model in \citep{Xie2020}. Compared to this convex counterpart, our nonconvex $\ell_{2,q}$ regularization effectively alleviates the inherent estimation bias from convex penalties, and yields stronger sparsity induction and better smoothness modeling, especially for recovering weak signals in high-dimensional settings.
\subsection{Theoretical Properties of \texorpdfstring{$\text{NS}^3\text{R}$}{NS3R}}
\nopagebreak[4]
\hspace{2em} To establish the theoretical properties of the $\text{NS}^3\text{R}$, we use the following notation and conventions throughout this paper. We assume $X$ is deterministic and normalized such that the diagonal entries of the Gram matrix $\Sigma := X^\top X / n$ are $1$. We use $\beta_{\mathcal{G}_j} = (\beta_{j,k} : k \in \{1, \dots, T_j\})$ to denote the sub-vector of $\beta$ formed by entries indexed in $\mathcal{G}_j$, $X_{\mathcal{G}_j}$ to denote the submatrix of $X$ formed by columns in $\mathcal{G}_j$, and $D_{\mathcal{G}_j}$ to denote the first-order difference matrix of the $j$-th group. In what follows, whenever it does not cause confusion, we abbreviate $\beta_{\mathcal{G}_j}$, $X_{\mathcal{G}_j}$, and $D_{\mathcal{G}_j}$ to $\beta_j$, $X_j$, and $D_j$, respectively. We define the group support set and support set of the true sparse signal $\beta^*$ as
\[
\mathbb{G} := \text{gsupp}(\beta^*) = \{1 \leq j \leq J : \beta_j^* \neq \mathbf{0}\},
\]
\[
\mathcal{S} := \bigcup_{j \in \mathbb{G}} \mathcal{G}_j,
\]
with the group sparsity index $s = |\mathbb{G}|$ and the sparsity index $s^* = |\mathcal{S}|$. We use $\beta_{\mathcal{S}} \in \mathbb{R}^{s^*}$ (and similarly $\beta_{\mathcal{S}^c} \in$$ \mathbb{R}^{p - s^*}$) to denote the sub-vector of $\beta \in \mathbb{R}^p$ consisting of entries indexed in $\mathcal{S}$ (or $\mathcal{S}^c$). Corresponding submatrix notation is used for $X$ and $D$. We also denote the maximum group size as $T_{\text{max}} = \max\{T_j : 1 \leq j \leq J\}$ and the minimum group size as $T_{\text{min}} = \min\{T_j : 1 \leq j \leq J\}$. We write $a \lesssim b$ if $a \leq c b$ for some constant $c > 0$, and $a = O(b)$ to indicate that $a$ and $b$ are of the same order up to a positive constant.

We now derive the global recovery bounds for $\text{NS}^3\text{R}$, and begin by stating the necessary conditions for our results.
\begin{definition}[$(2, q)$-Group Restricted Compatibility Condition ($(2, q)$-GRCC)]\label{def1}
	Let $0 < q \leq 1$ and $0 < \alpha \leq 1$, $X \in \mathbb{R}^{n \times p}$ is said to satisfy the $(2, q)$-Group Restricted Compatibility Condition for $\kappa > 0$ and $\gamma > 1$ if
	\begin{equation}
		\sum_{j \in \mathbb{G}} \sqrt{T_j} \|\Delta_j\|_2 \leq \frac{\sqrt{s^*} \|X\Delta\|_2}{\sqrt{n} \kappa}
		\tag{5}
		\label{eq:opt5}
	\end{equation}
	for all $\Delta \in \Delta_\gamma$, where
\begin{equation}
	\label{eq:opt6}
	\begin{split}
		\Delta_\gamma &:= \bigg\{ \Delta \in \mathbb{R}^p : \alpha \sum_{j \in \mathbb{G}^c} T_j^{q/2} \|\Delta_j\|_2^2 + (1 - \alpha) \sum_{j \in \mathbb{G}^c} T_j^{q/2} \|D_j \Delta_j\|_2^2 \\
		&\quad \leq \gamma \big( \alpha \sum_{j \in \mathbb{G}} T_j^{q/2} \|\Delta_j\|_2^2 + (1 - \alpha) \sum_{j \in \mathbb{G}} T_j^{q/2} \|D_j \Delta_j\|_2^2 \big) \bigg\}
	\end{split}
	\tag{6}
\end{equation}
\end{definition}
The $(2, q)$-GRCC extends the Compatibility Condition (CC) adopted by the group square-root Lasso \citep{Bunea2013} to vectors whose structure is characterized by a joint $\ell_{2,q}$ norm, and reduces to this original CC when $\alpha = 1$ and $q = 1$. Existing literature has established that the $(2,1)$-GRCC is much stronger than the $(2,q)$-GRCC for $q<1$, meaning the lower the $q$, the weaker the $(2, q)$-GRCC. The $(2, q)$-GRCC therefore covers a broader range of cases than the standard CC. See \citep{Hu2017, Bickel2009, Meinshausen2009, Li2023} and references therein for further discussion.

For the remainder of this paper, we assume that $\varepsilon_i$ are \text{i.i.d.} Gaussian random variables following $\mathcal{N}(0, \sigma^2)$, and that there exists a positive constant $\rho$ such that
\begin{equation}
	\|\beta^*\|_{2,\infty} + \left( \frac{1-\alpha}{\alpha} \right)^{1/q} \|D\beta^*\|_{2,\infty} \leq 2^{-2/q} \rho.
	\tag{7}
	\label{eq:opt7}
\end{equation}

We define two random events associated with the linear regression model \eqref{eq:opt1} as follows:
\begin{equation}
	\mathscr{A} = \left\{ \varepsilon: \frac{\|\varepsilon\|_2}{\sqrt{n}} \leq \sigma \sqrt{1+t} \right\},
	\tag{8}
	\label{eq:opt8}
\end{equation}
and
\begin{equation}
	\mathscr{B} = \left\{ (2\bar{\rho})^{1-q} V \leq \frac{\lambda_1}{\bar{\gamma}} \right\},
	\tag{9}
	\label{eq:opt9}
\end{equation}
where the noise component $V$ is given by
\begin{equation}
	V = \max_{1\leq j\leq J} \left\{ \frac{\sqrt{n} \|X_j^\top \varepsilon\|_2}{\sqrt{T_j} \|\varepsilon\|_2} \right\}.
	\tag{10}
	\label{eq:opt10}
\end{equation}
Here, $\bar{\rho} = \left( \frac{\sigma n \sqrt{1+t}}{\lambda_1} + \frac{2^{-q} s T_{\text{max}}^{q/2} \rho^q}{2} \right)^{1/q}$ and $\bar{\gamma} = \frac{\gamma+1}{\gamma-1}$ for $\gamma > 1$.

With $\kappa > 0$ and $\gamma > 1$ given by the $(2, q)$-GRCC above, we assume in what follows that the group sparsity index $s$ and the sparsity index $s^*$ satisfy
\begin{equation}
	s^{* q/2} s^{1-q} \lesssim \frac{n}{\lambda},
	\tag{11}
	\label{eq:opt11}
\end{equation}
where $\lambda = \lambda_1 + \lambda_2$. We will show in Lemma \ref{eventB} that, in terms of orders of magnitude, the value of $\lambda$ ensuring event $\mathscr{B}$ holds with high probability is bounded above by $\lambda = O\left( s^{\frac{1-q}{q}} \sqrt{n \log J} \right)$. Therefore, the condition \eqref{eq:opt11} is equivalent to
\begin{equation*}
	s^{* q/2} s^{1/q - q} \lesssim \sqrt{\frac{n}{\log J}}.
\end{equation*}
Intuitively, this means that when the number of groups $J$ greatly exceeds $n$, reliable performance can only be achieved if the true underlying model has substantially fewer parameters than $n$.

The following result demonstrates that the global solutions of the $\text{NS}^3\text{R}$ enjoy the dominant property with high probability.
\begin{pro}\label{pro1}
	Let $\hat{\beta}$ be a global solution of criterion \eqref{eq:opt4}, and assume that condition \eqref{eq:opt7} is satisfied. Then, on the event $\mathscr{A} \cap \mathscr{B}$, the following hold:
	\begin{equation}
		\sum_{j=1}^J \sqrt{T_j} \|\hat{\beta}_j - \beta_j^*\|_2 \leq (2\bar{\rho})^{1-q} \sum_{j=1}^J T_j^{q/2} \|\hat{\beta}_j - \beta_j^*\|_2^q
		\tag{12}
	   \label{eq:opt12}
	\end{equation}
	and
	\begin{equation}
		\begin{aligned}
			& \alpha \sum_{j \in \mathbb{G}^c} T_j^{q/2} \|\hat{\beta}_j - \beta_j^*\|_2^q + (1 - \alpha) \sum_{j \in \mathbb{G}^c} T_j^{q/2} \|D_j (\hat{\beta}_j - \beta_j^*)\|_2^q \\
			& \leq \gamma \bigg[ \alpha \sum_{j \in \mathbb{G}} T_j^{q/2} \|\hat{\beta}_j - \beta_j^*\|_2^q + (1 - \alpha) \sum_{j \in \mathbb{G}} T_j^{q/2} \|D_j (\hat{\beta}_j - \beta_j^*)\|_2^q \bigg]
		\end{aligned}
		\tag{13}
		\label{eq:opt13}
	\end{equation}
	where $\bar{\rho} = \left( \frac{\sigma n \sqrt{1+t}}{\lambda_1} + \frac{2^{-q} s T_{\text{max}}^{q/2} \rho^q}{2} \right)^{1/q}$.
\end{pro}

\begin{rmk}
	The result above implies that the estimation error $\Delta = \hat{\beta} - \beta^*$ lies in the set $\Delta_\gamma$ (the feasible set of the $(2, q)$-GRCC, cf. \Cref{def1}) whenever the event $\mathscr{A} \cap \mathscr{B}$ holds. By Lemma \ref{eventA} and Lemma \ref{eventB}, we have $\mathbb{P}(\mathscr{A}) \geq 1 - \delta_1$ and $\mathbb{P}(\mathscr{B}) \geq 1 - \delta_2$ under our assumptions, so that $\mathbb{P}(\mathscr{A} \cap \mathscr{B}) \geq 1 - \delta_1 - \delta_2$. This ensures $\Delta \in \Delta_\gamma$ with probability at least $1 - \delta_1 - \delta_2$, rendering the subsequent theoretical results valid in a high-probability sense rather than just conditional on $\mathscr{A} \cap \mathscr{B}$.
\end{rmk}
With the critical regularity condition established, we now state the main theorem of this section, which establishes the prediction and estimation properties of the NS$^3$R under the weak $(2,q)$-GRCC. This result generalizes [\citep{Bunea2013}, Theorem II.1], where the group square-root Lasso is treated, corresponding in our set-up to the special case $q=1$.
\begin{thm}[Oracle Inequality]\label{thm:oracle_inequality}
	Assume that $X$ satisfies the $(2, q)$-GRCC with $\kappa > 0$ and $\gamma > 1$, and that \eqref{eq:opt11} holds. Under the assumptions of Lemmas \ref{eventA}, \ref{eventB}, and \Cref{pro1}, we choose
	\begin{equation}
		\lambda_1 \geq \frac{n}{\sqrt{n-T_{\text{max}}}} T_{\text{max}}^{\frac{1-q}{2}} \sqrt{2\zeta} s^{\frac{1-q}{q}} \rho^{1-q} \bar{\gamma} \left( 1 + \sqrt{\frac{2\log(2J/\delta_2)}{T_{\text{min}}}} \right),
		\tag{14}
		\label{eq:opt14}
	\end{equation}
	where $\zeta = \max_j \|X_j\|_2^2 / n$. Then, with probability at least $1 - \delta_1 - \delta_2$, we have
	\begin{equation}
		\frac{1}{\sqrt{n}} \|X(\hat{\beta} - \beta^*)\|_2 \lesssim \left( \sigma \frac{\lambda s^{1-q} s^{* q/2}}{n \kappa^q} \right)^{\frac{1}{2-q}}
		\tag{15}
		\label{eq:opt15}
	\end{equation}
	and
	\begin{equation}
		\sum_{j=1}^J T_j^{q/2} \|\hat{\beta}_j - \beta_j^*\|_2^q \lesssim \left( \sigma \frac{\lambda s^{2(1/q - 1)} s^*}{n \kappa^2} \right)^{\frac{q}{2-q}}.
		\tag{16}
		\label{eq:opt16}
	\end{equation}
\end{thm}
\begin{rmk}	
(i) \Cref{thm:oracle_inequality} establishes oracle inequalities for the NS$^3$R estimator under the weak $(2,q)$-GRCC, which stands as one of the mildest regularity conditions in the literature. As established in prior work \citep{Hu2017, Li2023, Negahban2012}, the compatibility condition and its related variants in the literature (e.g., q-restricted eigenvalue condition, group restricted eigenvalue condition) are strictly weaker than the Restricted Isometry Property (RIP) and are compatible with a wide range of design matrices, including Gaussian/sub-Gaussian random designs and grouped/block-wise correlated designs. Though the convergence rates for $q < 1$ are not superior to those of the classical $\ell_{2,1}$ group Lasso, the required $(2,q)$-GRCC is substantially weaker: for problems where the $(2,q)$-GRCC holds but the standard group restricted eigenvalue condition fails, the group Lasso may yield inconsistent estimates, while our NS$^3$R estimator still maintains valid recovery guarantees.

(ii) Notably, the tuning parameters $\lambda_1$ and $\lambda_2$ are $\sigma$-free, i.e., independent of the unknown noise level $\sigma$, a critical advantage in high-dimensional settings where $\sigma$ is hard to estimate. In particular, when the maximum group size $T_{\max} = O(1)$, the estimator attains convergence rates of $O\left((\sigma \lambda / n)^{\frac{1}{2-q}} \sqrt{s}\right)$ for prediction and $O\left((\sigma \lambda / n)^{\frac{q}{2-q}} s\right)$ for estimation, which coincide with the optimal rates for the $\ell_{p,q}$ regularization problem discussed in \citep{Hu2017} under the noise-free case. Unlike the penalized least-squares regularization that requires prior knowledge of $\sigma$ for tuning, our $\sigma$-independent framework avoids pre-estimating the noise level, a well-known challenge in ultra-high-dimensional settings. In practice, we can fix $\lambda_1$ at its theoretical value and set $\lambda_2 = \theta \lambda_1$ for a pre-specified constant $\theta > 0$, or adaptively select the ratio $\theta$ via cross-validation (CV) or information criteria such as AIC or BIC to optimize predictive performance.
\end{rmk}

\section{Algorithm and global convergence analysis}\label{sec:algorithm}
\subsection{Subdifferential and optimality conditions}
\hspace{2em} In this section, we establish the subdifferential (see Appendix \ref{app:subdiff} for more details) and optimality conditions for the NS$^3$R, both of which are critical for constructing an efficient iterative algorithm and analyzing its convergence properties. We begin with an equivalent reformulation of the original model \eqref{eq:opt4} as follows
\[
\min_{\beta \in \mathbb{R}^p} F(\beta) = \|y - X\beta\|_2 + \lambda_1 \sum_{j=1}^J T_j^{q/2} \|\beta_j\|_2^q + \lambda_2 \sum_{j=1}^J T_j^{q/2} \|D_j \beta_j\|_2^q
\tag{17}
\label{eq:opt17}
\]

The following proposition gives the explicit form of the subdifferential of $F(\beta)$, which is critical for deriving optimality conditions and analyzing the convergence of our algorithm.

\begin{pro}\label{pro2}
	Assume that $y \neq X\beta$ over some set $\mathcal{C}$. Then, on $\mathcal{C}$, the subdifferential of $F(\beta)$ at $\beta$ is given groupwise for each group $j \in \{1, \dots, J\}$ by:
\begin{equation}
	(\partial F(\beta))_j =
	\begin{cases}
		\begin{aligned}
			&-\frac{X_j^\top (y - X\beta)}{\|y - X\beta\|_2} + \lambda_1 T_j^{q/2} q\|\beta_j\|_2^{q-2} \beta_j \\
			&\quad + \lambda_2 T_j^{q/2} q\|D_j \beta_j\|_2^{q-2} D_j^\top D_j \beta_j,
		\end{aligned}
		& \beta_j \neq 0 \text{ and } D_j \beta_j \neq 0, \\[1em]
		(-\infty, +\infty)^{T_j}, & \text{otherwise},
	\end{cases}
	\tag{18}
	\label{eq:opt18}
\end{equation}
	where $X_j \in \mathbb{R}^{n \times T_j}$ denotes the submatrix of $X$ corresponding to group $j$ with cardinality $T_j$, and $0 < q < 1$.
\end{pro}

In the framework of generalized subdifferentials, the well-known Fermat’s rule remains valid, that is, if $\beta \in \mathbb{R}^p$ is a local minimizer of $F(\beta)$ then
\[
0 \in \partial F(\beta).
\tag{19}
\label{eq:opt19}
\]
A point $\beta \in \mathbb{R}^p$ that satisfies \eqref{eq:opt19} is referred to as a stationary point of $F(\beta)$. For the optimization problem \eqref{eq:opt17}, if $\beta$ is a stationary point, then the componentwise optimality condition
\[
-\frac{X_j^\top (y - X\beta)}{\|y - X\beta\|_2} + \lambda_1 T_j^{q/2} q \|\beta_j\|_2^{q-2} \beta_j + \lambda_2 T_j^{q/2} q \|D_j \beta_j\|_2^{q-2} D_j^\top D_j \beta_j = 0
\tag{20}
\]
holds for all groups $j$ satisfying $\beta_j \neq 0$ and $D_j \beta_j \neq 0$. Further details can be found in Section 3.1 of \citep{Xue2020}.

\subsection{Double-screened IRLS with thresholding algorithm}
\hspace{2em} The objective function $F(\beta)$ is nonconvex, nonsmooth, and non-Lipschitz, so most existing nonconvex algorithms either fail to guarantee convergence, or only ensure subsequential convergence under stringent conditions. Motivated by the IRLS with thresholding algorithm in \citep{Feng2020}, we propose a double-screened IRLS with thresholding (DS-IRLS-th) algorithm to solve the square-root optimization problem with a joint nonconvex $\ell_{2,q}$ regularizer. Our algorithm adopts a majorization-minimization framework to iteratively solve a sequence of reweighted least squares problems, and incorporates a double-screening mechanism on both group-level coefficient magnitudes and within-group difference magnitudes. In practice, the DS-IRLS-th algorithm maintains two key sets at each iteration: the group support set $\mathbb{G}^k$, including groups with coefficients exceeding the threshold $\tau_1$; and the difference active set $\mathbb{A}^k$, a subset of $\mathbb{G}^k$ containing groups where within-group differences exceed the threshold $\tau_2$. This nested structure guarantees a support-shrinking property across iterations, a critical requirement for proving convergence.

More precisely, let $\beta^k$ be the minimizer at the $k$-th iteration. By the nested double-screening mechanism, we only optimize the objective over the group support set $\mathbb{G}^k$, with coefficients for groups outside $\mathbb{G}^k$ fixed to zero. For groups in $\mathbb{G}^k\setminus\mathbb{A}^k$, all coefficients within the same group are constrained to a constant value. Equivalently, we solve the following constrained optimization problem at the $(k+1)$th iteration:
\[
\tilde{\beta}^{k+1} = \arg\min_{\beta \in \mathbb{R}^p} \left\{ \frac{1}{r_k} \|y - X\beta\|_2^2 + \lambda_1 \sum_{j \in \mathbb{G}^k} T_j^{q/2} w_{1,j}^k \|\beta_j\|_2^2 + \lambda_2 \sum_{j \in \mathbb{A}^k} T_j^{q/2} w_{2,j}^k \|D_j \beta_j\|_2^2 \right\},
\tag{21}
\label{eq:opt21}
\]
\[
\text{s.t.} \quad \beta_j = 0, \ \forall j \notin \mathbb{G}^k, \qquad D_j \beta_j = 0, \ \forall j \in \mathbb{G}^k \setminus \mathbb{A}^k
\]
where $r_k = \|y - X\beta^k\|_2$ is the current residual norm, and the weights are defined by
\[
w_{1,j}^k = q\|\beta_j^k\|_2^{q-2}, \ \forall j \in \mathbb{G}^k; \quad w_{2,j}^k = q\|D_j\beta_j^k\|_2^{q-2}, \ \forall j \in \mathbb{A}^k.
\]

Our algorithm is summarized as follows.

\vspace{1em}
\hrule
\vspace{0.3em}

\noindent\textbf{Algorithm 1} DS-IRLS-th: Double-screened IRLS with Thresholding Algorithm
\addcontentsline{loa}{algorithm}{\protect\numberline{\thealgorithm}DS-IRLS-th}
\label{alg:ds-irls-th}

\vspace{0.3em}
\hrule
\vspace{0.5em}

\noindent\hspace{1em}\textbf{Input}: Threshold $\tau_1, \tau_2 > 0$, initial point $\beta^0 \in \mathbb{R}^p$ satisfying $\|\beta_j^0\|_2 \geq \tau_1$ and $\|D_j\beta_j^0\|_2 \geq \tau_2$ for all $1 \leq j \leq J$.

\vspace{0.6em}
\noindent\hspace{1em}\textbf{Initialize}: Set the group support set $\mathbb{G}^0 = \{j: \|\beta_j^0\|_2 \geq \tau_1\}$.

\vspace{0.6em}
\noindent\hspace{1em}\textbf{Iteration}: For $k = 0,1,2,\dots$ do

\vspace{0.5em}
\noindent\hspace{2em}1. Compute the residual norm: $r_k = \|y - X\beta^k\|_2$.

\vspace{0.5em}
\noindent\hspace{2em}2. Set the difference active set: $\mathbb{A}^k = \{j \in \mathbb{G}^k: \|D_j\beta_j^k\|_2 \geq \tau_2\}$.

\vspace{0.5em}
\noindent\hspace{2em}3. Compute the weights for each group in the group support set:
\vspace{-0.2em}
\[
w_{1,j}^k = q\|\beta_j^k\|_2^{q-2}, \ \forall j \in \mathbb{G}^k; \quad w_{2,j}^k = q\|D_j\beta_j^k\|_2^{q-2}, \ \forall j \in \mathbb{A}^k.
\]

\vspace{0.5em}
\noindent\hspace{2em}4. Solve the following constrained optimization problem:
\vspace{-0.2em}
\[
\tilde{\beta}^{k+1} = \arg\min_{\beta \in \mathbb{R}^p} \left\{ \frac{1}{r_k} \|y - X\beta\|_2^2 + \lambda_1 \sum_{j \in \mathbb{G}^k} T_j^{q/2} w_{1,j}^k \|\beta_j\|_2^2 + \lambda_2 \sum_{j \in \mathbb{A}^k} T_j^{q/2} w_{2,j}^k \|D_j\beta_j\|_2^2 \right\},
\]
\vspace{-0.2em}
\[
\text{s.t. } \beta_j = 0, \forall j \notin \mathbb{G}^k, \quad D_j\beta_j = 0, \forall j \in \mathbb{G}^k \setminus \mathbb{A}^k.
\]

\vspace{0.5em}
\noindent\hspace{3em}or equivalently, denote $\mathcal{S}^k = \cup_{j \in \mathbb{G}^k} \mathcal{G}_j$, and define the group-wise transformation matrix:
\vspace{-0.2em}
\[
P_j^k =
\begin{cases}
	I_{T_j}, & j \in \mathbb{A}^k \\
	\mathbf{1}_{T_j}, & j \in \mathbb{G}^k \setminus \mathbb{A}^k
\end{cases}
\]
\vspace{-0.2em}
\noindent\hspace{3em}
\begin{minipage}{\dimexpr\linewidth - 3em\relax}
	where $T_j$ is the cardinality of the $j$-th group, $I_{T_j}$ is the $T_j$-dimensional identity matrix, and $\mathbf{1}_{T_j}$ is the $T_j$-dimensional vector with all components equal 1. Set $P = \mathrm{diag}(P_j^k, j \in \mathbb{G}^k)$, $\tilde{X}_{\mathcal{S}^k} = X_{\mathcal{S}^k} P$ and block-diagonal weight matrix $W^k = \mathrm{diag}(W_j^k, j \in \mathbb{G}^k)$ with
\end{minipage}
\vspace{-0.2em}
\[
W_j^k =
\begin{cases}
	\lambda_1 T_j^{q/2} w_{1,j}^k I_{T_j} + \lambda_2 T_j^{q/2} w_{2,j}^k D_j^T D_j, & j \in \mathbb{A}^k \\
	\lambda_1 T_j^{q/2} w_{1,j}^k \cdot T_j, & j \in \mathbb{G}^k \setminus \mathbb{A}^k
\end{cases}
\]

\vspace{0.5em}
\noindent\hspace{3em}Solve the following linear system:
\vspace{-0.2em}
\[
\left( \frac{1}{r_k} \tilde{X}_{\mathcal{S}^k}^T \tilde{X}_{\mathcal{S}^k} + W^k \right) \beta^{red} = \frac{1}{r_k} \tilde{X}_{\mathcal{S}^k}^T y
\]
\vspace{-0.2em}
\noindent\hspace{3em}where $\beta^{red}$ is the reduced variable, and the original coefficient is recovered by $\tilde{\beta}_{\mathcal{S}^k}^{k+1} = P \beta^{red}$.

\vspace{0.5em}
\noindent\hspace{2em}5. Thresholding: $\beta_j^{k+1} =
\begin{cases}
	0, & \|\tilde{\beta}_j^{k+1}\|_2 < \tau_1 \\
	\tilde{\beta}_j^{k+1}, & \text{otherwise}
\end{cases}$.

\vspace{0.5em}
\noindent\hspace{2em}6. Update the group support set: $\mathbb{G}^{k+1} = \{j: \|\beta_j^{k+1}\|_2 \geq \tau_1\}$.

\vspace{0.5em}
\noindent\hspace{1em}\textbf{end for}

\vspace{0.6em}
\noindent\hspace{1em}\textbf{Output}: The sequence $\{\beta^k\}$.
\vspace{0.3em}
\hrule

\begin{rmk}	
(i) The nested double-screening mechanism of the DS-IRLS-th algorithm restricts the difference active set $\mathbb{A}^k$ to groups with both non-negligible coefficients and non-flat within-group differences. This avoids the numerical instability caused by excessively large weights from near-zero coefficients or within-group differences in the reweighted least squares step. Importantly, all computations are restricted to the shrinking group support set, which significantly reduces the dimension of the linear system solved at each iteration. Specifically, at the $(k+1)$th iteration, we solve the reduced system
$$
\left( \frac{1}{r_k} \tilde{X}_{\mathcal{S}^k}^T \tilde{X}_{\mathcal{S}^k} + W^k \right) \beta^{red} = \frac{1}{r_k} \tilde{X}_{\mathcal{S}^k}^T y
$$
where $\mathcal{S}^k = \cup_{j \in \mathbb{G}^k} \mathcal{G}_j$ denotes the support set, $\tilde{X}_{\mathcal{S}^k} = X_{\mathcal{S}^k} P$ is the reduced design matrix with transformation $P$, and $W^k$ is the block-diagonal weight matrix constructed from the estimate at the $k$-th iteration.

In early iterations, the relatively large support set makes direct matrix inversion computationally costly. For this case, we recommend the conjugate gradient method in \citep{Saad2003}, which is well-suited for large symmetric positive-definite systems. As the two nested sets $\mathbb{G}^k$ and $\mathbb{A}^k$ shrink rapidly and monotonically across iterations, the computational cost for solving the reduced linear system decreases substantially, giving our method a clear efficiency advantage over existing alternatives, especially for large-scale sparse recovery problems.

(ii) The DS-IRLS-th algorithm extends the IRLS-th framework in \citep{Feng2020} by incorporating a double-screening mechanism designed for grouped and structured smoothness priors. This extension improves both sparse recovery accuracy and computational efficiency, driven by the rapid, monotonic shrinkage of the support set across iterations. A key advantage of this double-screening design is that it identifies pseudo-active groups, which have nearly flat within-group coefficients, and imposes equality constraints on all coefficients within these groups. This avoids numerical instability caused by near-zero within-group differences in the reweighting step, while also greatly reducing the effective dimension of the optimization problem. For signals with inherent grouped and locally smooth structures, this leads to more reliable recovery of the true sparse pattern, and mitigates overfitting in smooth regions more effectively.
\end{rmk}

\subsection{Global convergence analysis}
\hspace{2em} In this subsection, we establish the global convergence result of the sequence $\{\beta^k\}$ generated by the DS-IRLS-th algorithm. Our convergence analysis relies on the Kurdyka--\L{}ojasiewicz (KL) inequality framework, which requires the objective function $F(\beta)$ to satisfy the KL property. For completeness, we include the formal definitions of the KL property and semi-algebraic functions in Appendix \ref{app:kl_property}, both of which are classical results in nonconvex and nonsmooth optimization. We also verify in Appendix \ref{app:kl_property} that our objective $F(\beta)$ belongs to the class of semi-algebraic functions via standard closure properties of semi-algebraic sets and functions, which ensures that $F(\beta)$ satisfies the KL property at all points relevant to our analysis. We prove convergence by verifying three key conditions for the KL framework, i.e., that $\{\beta^k\}$ is bounded, that the function value sequence $\{F(\beta^k)\}$ admits the sufficient descent property, and that the subdifferential sequence $\{\partial F(\beta^k)\}$ is relatively bounded. Notably, our result holds without the proximal term required in most existing KL-based convergence analyses.

To apply the KL inequality framework, we verify the required regularity conditions in the two lemmas below. \Cref{lem:boundedness_sufficient_descent} characterizes the boundedness of the iterate sequence and the sufficient descent property of the objective, while \Cref{lem:subdifferential_boundedness} establishes the relative boundedness of the subdifferential sequence.

Let $l_f$ denote the local Lipschitz constant of the objective function $F(\beta)$ around the KKT point, and let $\sigma_{\min}(D_j^\top)$ denote the minimum singular value of the matrix $D_j^\top$. We now state our first lemma.
\begin{lem}[Boundedness and sufficient descent of the objective function value]
	\label{lem:boundedness_sufficient_descent}
	Let $\{\beta^k\}$ be the sequence generated by the DS-IRLS-th algorithm. The following statements hold:
	\begin{enumerate}[label=(\roman*), font=\upshape,
		topsep=0pt,
		itemsep=0pt,
		parsep=0pt,
		partopsep=0pt,
		leftmargin=*
		]
		\item Suppose the threshold satisfies $\tau_2 \leq \tilde{\tau}_2 = \left( \lambda_2 \sigma_{\min}(D_j^\top) T_{\text{min}}^{q/2} / l_f \right)^{\frac{1}{1-q}}$, then there exists a constant $C > \tau_1$ such that
		\begin{equation}
			C \geq \|\beta_j^k\|_2, \quad \forall 1 \leq j \leq J,\ k \geq 0;
			\tag{22}
			\label{eq:opt22}
		\end{equation}
		\item Suppose the thresholds satisfy $\tau_1 \leq \tilde{\tau}_1 = \min\big\{ \big( (T_{\text{min}}^{q/2} \lambda_1) / (2\|X\|_2) \big)^{\frac{1}{1-q}}, \left( \lambda_1 / (2\tilde{a}) \right)^{\frac{1}{2-q}} \big\}$, and $\tau_2 \leq \tilde{\tau}_2$, then there exists a constant $a > 0$ such that
		\begin{equation}
			F(\beta^k) - F(\beta^{k+1}) \geq a \|\beta^k - \beta^{k+1}\|_2^2,
			\tag{23}
			\label{eq:opt23}
		\end{equation}
		holds for all $k \geq 0$.
	\end{enumerate}
\end{lem}

With the boundedness of the iterate sequence and the sufficient descent property of the objective established, we next characterize the relative boundedness of the subdifferential sequence.
\begin{lem}[Relatively bounded subdifferential]
	\label{lem:subdifferential_boundedness}
	For each $k \geq 0$, suppose that $\|y - X\beta^k\|_2 > 0$ is uniformly bounded from below. Then there exists $u^{k+1} \in \partial F(\beta^{k+1})$ such that
	\begin{equation}
		\|u^{k+1}\|_2 \leq b \|\beta^k - \beta^{k+1}\|_2,
		\tag{24}
		\label{eq:opt24}
	\end{equation}
	where $b$ is a positive constant.
\end{lem}

With the above two core properties of the DS-IRLS-th algorithm, we now establish its global convergence and finite-length property in the following theorem.
\begin{thm}[Global convergence and finite-length property]
	\label{thm:global_convergence}
	Let $\{\beta^k\}$ be the sequence generated by the DS-IRLS-th algorithm. For each $k \geq 0$, suppose that $\|y - X\beta^k\|_2 > 0$ is uniformly bounded from below, and the thresholds satisfy $\tau_1 \leq \tilde{\tau}_1$, $\tau_2 \leq \tilde{\tau}_2$ ($\tilde{\tau}_1, \tilde{\tau}_2$ are defined in \Cref{lem:boundedness_sufficient_descent}). The following statements hold:
	\begin{enumerate}[label=(\roman*), font=\upshape,
		topsep=0pt,
		itemsep=0pt,
		parsep=0pt,
		partopsep=0pt,
		leftmargin=*
		]
		\item The sequence $\{\beta^k\}$ converges to a stationary point $\beta^*$ of the objective function $F(\beta)$, i.e.,
		\begin{equation}
		\lim_{k\to\infty} \beta^k = \beta^*;
		\tag{25}
		\label{eq:opt25}
		\end{equation}
		\item The sequence $\{\beta^k\}$ has finite length, i.e.,
		\begin{equation}
			\sum_{k=0}^\infty \|\beta^{k+1} - \beta^k\|_2 < \infty.
			\tag{26}
			\label{eq:opt26}
		\end{equation}
	\end{enumerate}
\end{thm}

From Theorem \ref{thm:global_convergence}, the proposed DS-IRLS-th algorithm globally converges to a stationary point of $F(\beta)$ under mild regularity conditions. By contrast, most existing nonconvex algorithms for group-sparse regression only guarantee local (subsequence) convergence. This result holds for all iterations under the threshold conditions specified in \Cref{lem:boundedness_sufficient_descent}.

\section{Numerical experiments}\label{sec:numerical}
\hspace{2em} Before investigating the finite-sample performance of the NS$^3$R, we detail the parameter selection procedure for our algorithm. All numerical experiments are implemented in R and run on a personal desktop (Intel Core i7-14700K, 3.40 GHz, 64 GB of RAM).

We describe the general experimental setup below. We set the sample size $n = 250$, the dimension $p = 1000$, the number of groups $J = 100$, and the size of each group $T_j = 10$ for all $1 \leq j \leq J$. The sparsity level is defined as $s/J$, where $s$ denotes the number of active groups in the true $\beta^*$, and we vary $s$ across experiments. The coefficients within each active group are set to a constant value of 3, while all coefficients in inactive groups are fixed to 0. The covariates $(X_1, X_2, \dots, X_p)$ are independently generated from the standard normal distribution $N(0,1)$, and the response vector $y$ is generated by
\[
y = X\beta^* + \sigma \cdot \text{randn}(n,1),
\]
where $\sigma$ is the standard deviation of the Gaussian noise. We fix $\sigma = 0.01$ in all experiments.

In our implementation, we first set the thresholds $\tau_1 = 10^{-2}$, $\tau_2 = 10^{-4}$, and the exponent $q = 0.5$ for the proposed DS-IRLS-th algorithm. The algorithm is initialized with the ridge regression estimate, where the ridge penalty parameter is set to 0.01. For the regularization parameters, we adopt a modified version of the theoretical $\lambda$-selection strategy proposed in Section IV of \citep{Bunea2013}. By the argument in \citep{Bunea2013}, if $\lambda_1 \propto \sqrt{\frac{n\zeta\tau_0}{T_{\text{min}}\tau_0 + n - T_{\text{max}}}}$, where
\[
\tau_0 = F_{T_{\text{min}}, n-T_{\text{max}}}^{-1} \left(1 - \alpha_{\text{sig}}/J\right), \quad \zeta = \max_j \frac{1}{n}\|X_j\|_2^2,
\]
with $\alpha_{\text{sig}}$ a pre-specified significance level, then the noise component $V$ is guaranteed to be dominated by $\lambda_1$. Our preliminary experiments find that a global scaling factor $c = 0.2$ yields stable performance across different sparsity levels. We thus set
\begin{equation}\label{eq:opt27}
	\lambda_1 = c \sqrt{\frac{n\zeta\tau_0}{T_{\text{min}}\tau_0 + n - T_{\text{max}}}}.
	\tag{27}
\end{equation}
For the second regularization parameter, we set $\lambda_2 = \theta\lambda_1$, where the multiplier $\theta > 0$ is selected adaptively via the EBIC criterion \citep{Chen2008} over the predefined grid $\{2^{-6}, 2^{-5}, \dots, 2^0\}$. A fixed value of $\theta$ (e.g., $\theta=1$) may also be used as a simpler alternative.

For all experiments, we use the stopping criterion
$$
\frac{\|\beta^{k+1} - \beta^k\|_2}{\|\beta^k\|_2} < 10^{-4}.
$$

For each sparsity level, we independently generate 100 datasets $(X, y, \beta^*)$ under the above setup, then run the algorithm on each dataset and average the results across 100 trials. We evaluate model selection performance using the metrics below.

The first metric we consider is the group-level false positive rate (FPR), defined as
$$
\frac{|\{ j: \beta_j^* = 0, \hat{\beta}_j \neq 0 \}|}{|\{ j: \beta_j^* = 0 \}|},
$$
where $\beta_j^*$ and $\hat{\beta}_j$ denote the true and estimated coefficient vectors for the $j$-th group, respectively.

The second metric is the success rate, defined as the proportion of the 100 trials where
$$
\frac{\|\hat{\beta} - \beta^*\|_2}{\|\beta^*\|_2} < 0.01.
$$

\subsection{Choice of \texorpdfstring{$\tau_1$, $\tau_2$}{tau1, tau2}}
\hspace{2em} In this subsection, we discuss the choice of $\tau_1, \tau_2$. Following the discussion in Section 7.2 of \cite{Feng2020}, we fix $\tau_1 = 10^{-2}$. In our analysis, we test our algorithm with different $\tau_2$ values from the set $\{10^{-6}, 10^{-4}, 10^{-2}, 10^{-1}, 1\}$ across different sparsity levels. The success rates and average running times (over 100 independent trials) are given in \autoref{fig:choicetau}.

The results show that our algorithm with $\tau_2 = 10^{-6}$ achieves the optimal balance between model selection accuracy and computational efficiency. Therefore, we fix $\tau_1 = 10^{-2}$ and $\tau_2 = 10^{-6}$ for all subsequent experiments.
\begin{figure}[!t]
	\centering
	\begin{minipage}[t]{0.49\linewidth}
		\centering
		\includegraphics[width=\linewidth]{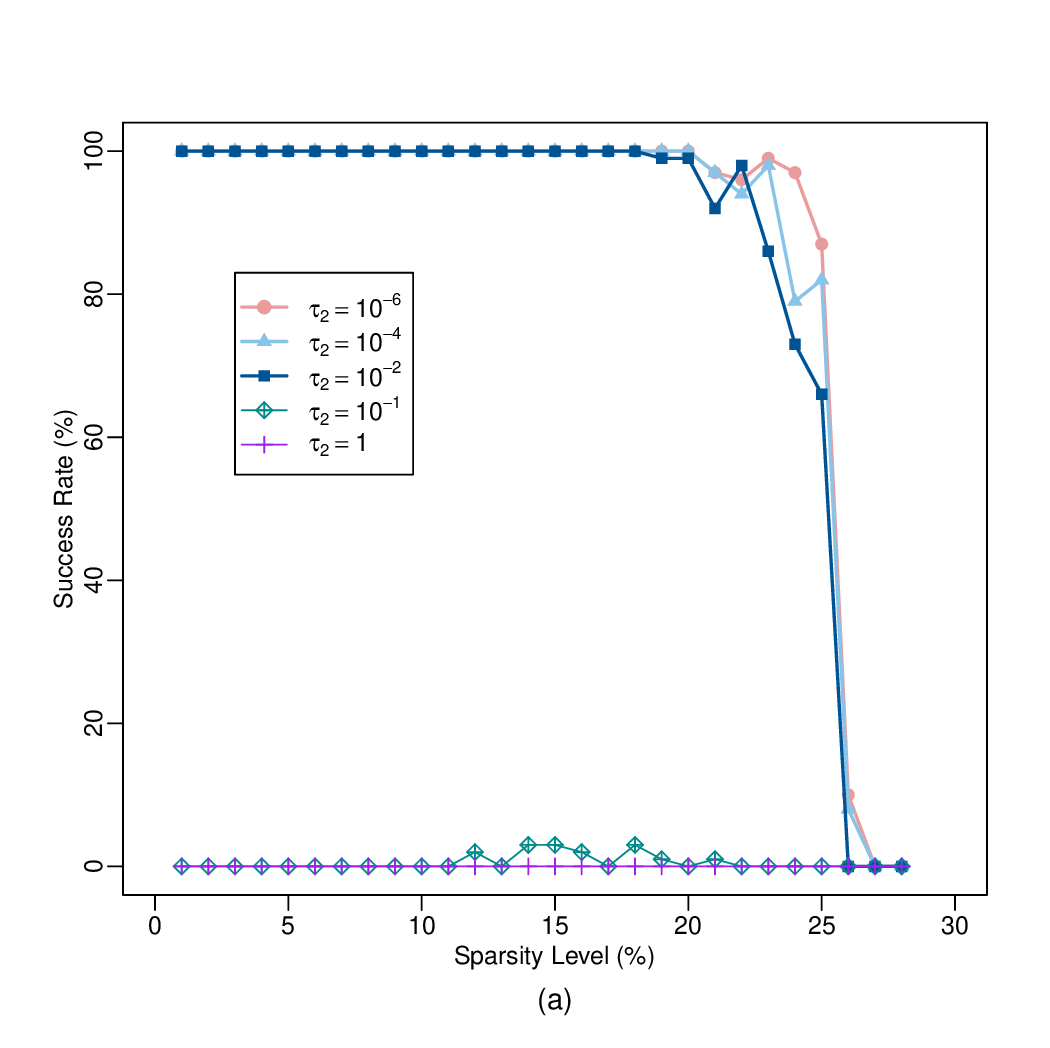} 
	\end{minipage}
	\hspace{0pt} 
	\begin{minipage}[t]{0.49\linewidth}
		\centering
		\includegraphics[width=\linewidth]{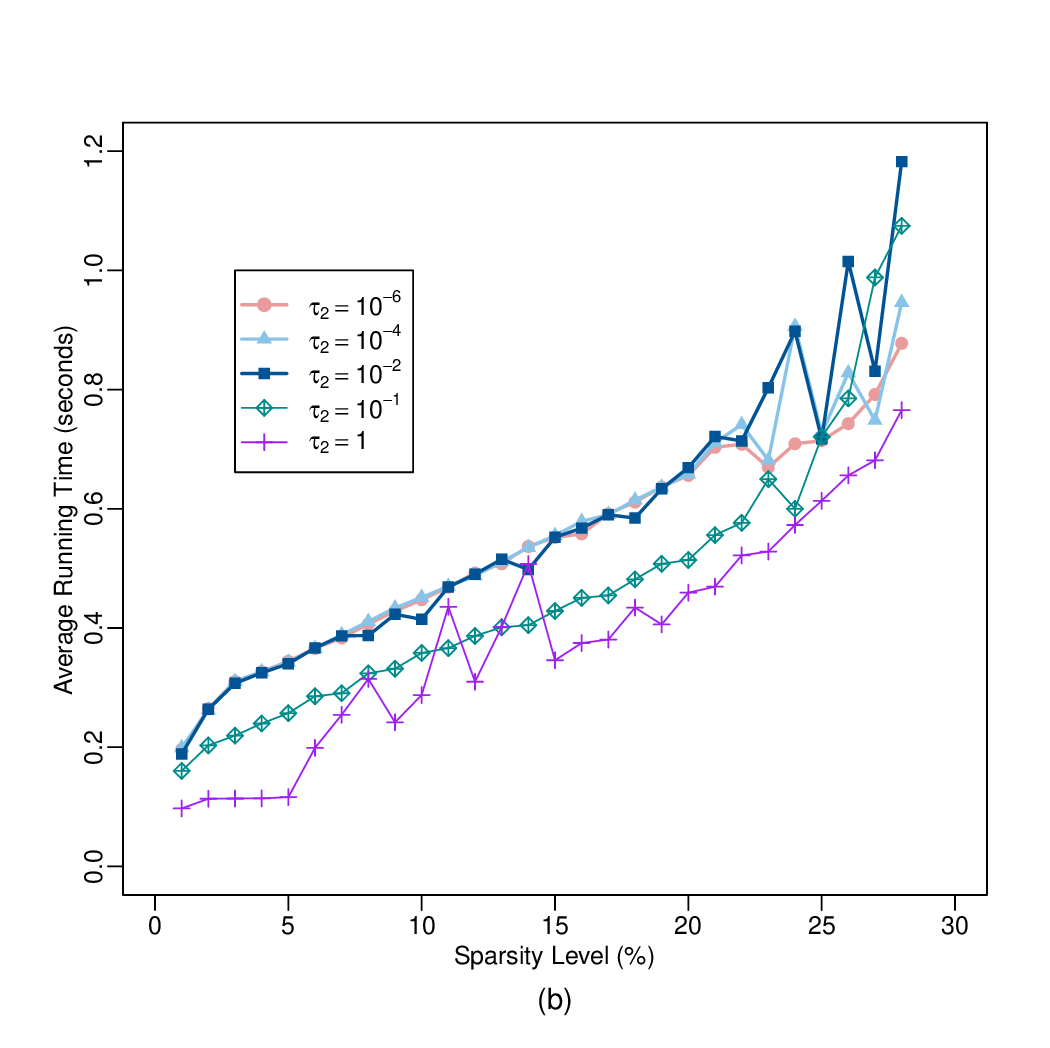}
	\end{minipage}

	\caption{The success rates and average running times of DS-IRLS-th algorithm with different threshold $\tau_2$.}
	\label{fig:choicetau}
\end{figure}

\subsection{Choice of \texorpdfstring{$q$}{q}}
\hspace{2em} In this subsection, we compare the performance of the DS-IRLS-th algorithm with different $q$ values from the set $\{0.1, 0.3, 0.5, 0.7, 0.9\}$ across different sparsity levels. The resulting success rates and false positive rates are summarized in \autoref{fig:choiceq}, based on 100 independent replications.

The algorithm performs the worst when $q = 0.9$. This poor performance stems from insufficient non-convexity, which fails to effectively promote sparsity. In contrast, the algorithm achieves the best overall performance when $q = 0.5$, which aligns with the simulation results in \cite{Feng2020}. We thus fix $q = 0.5$ for all subsequent experiments.

\begin{figure}[!t]
	\centering
	\begin{minipage}[t]{0.49\linewidth}
		\centering
		\includegraphics[width=\linewidth]{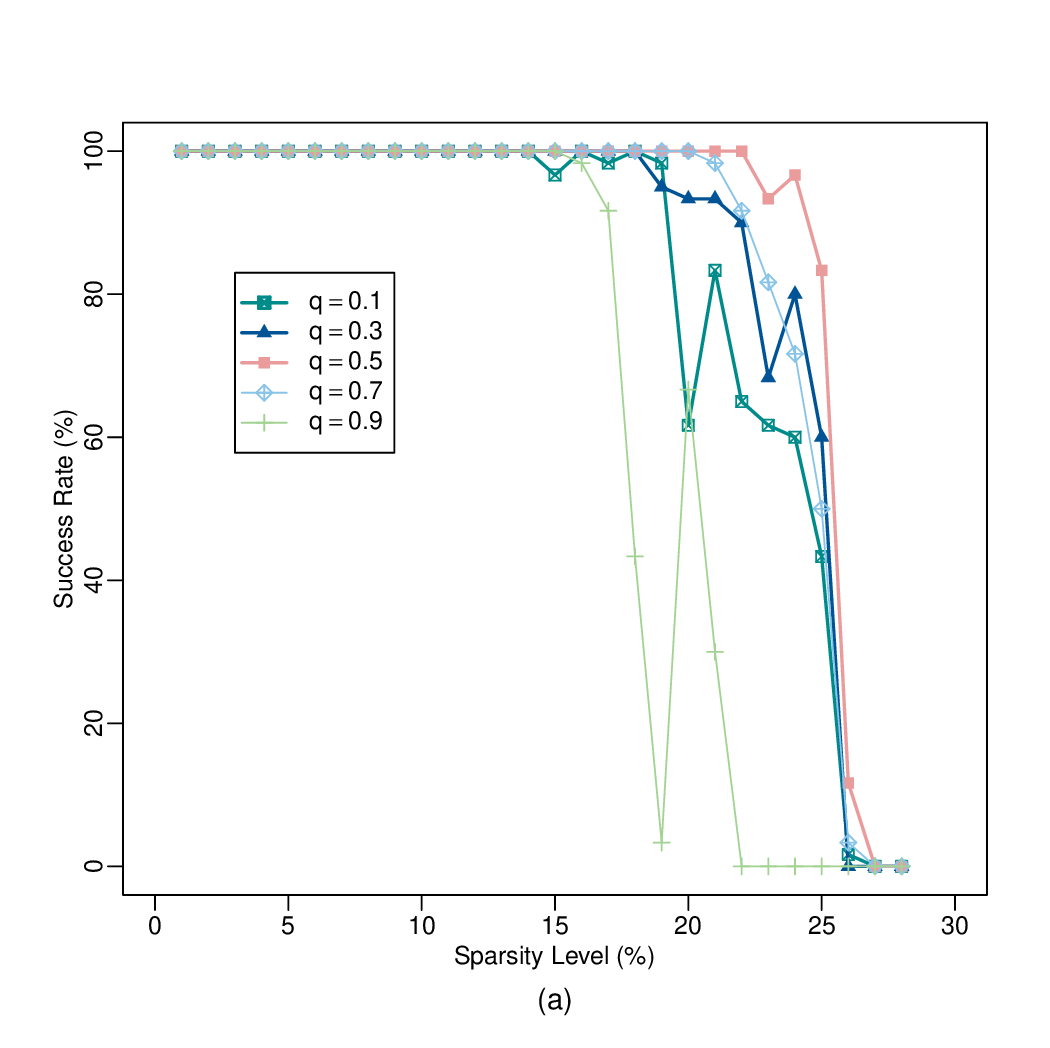} 
	\end{minipage}
	\hspace{0pt} 
	\begin{minipage}[t]{0.49\linewidth}
		\centering
		\includegraphics[width=\linewidth]{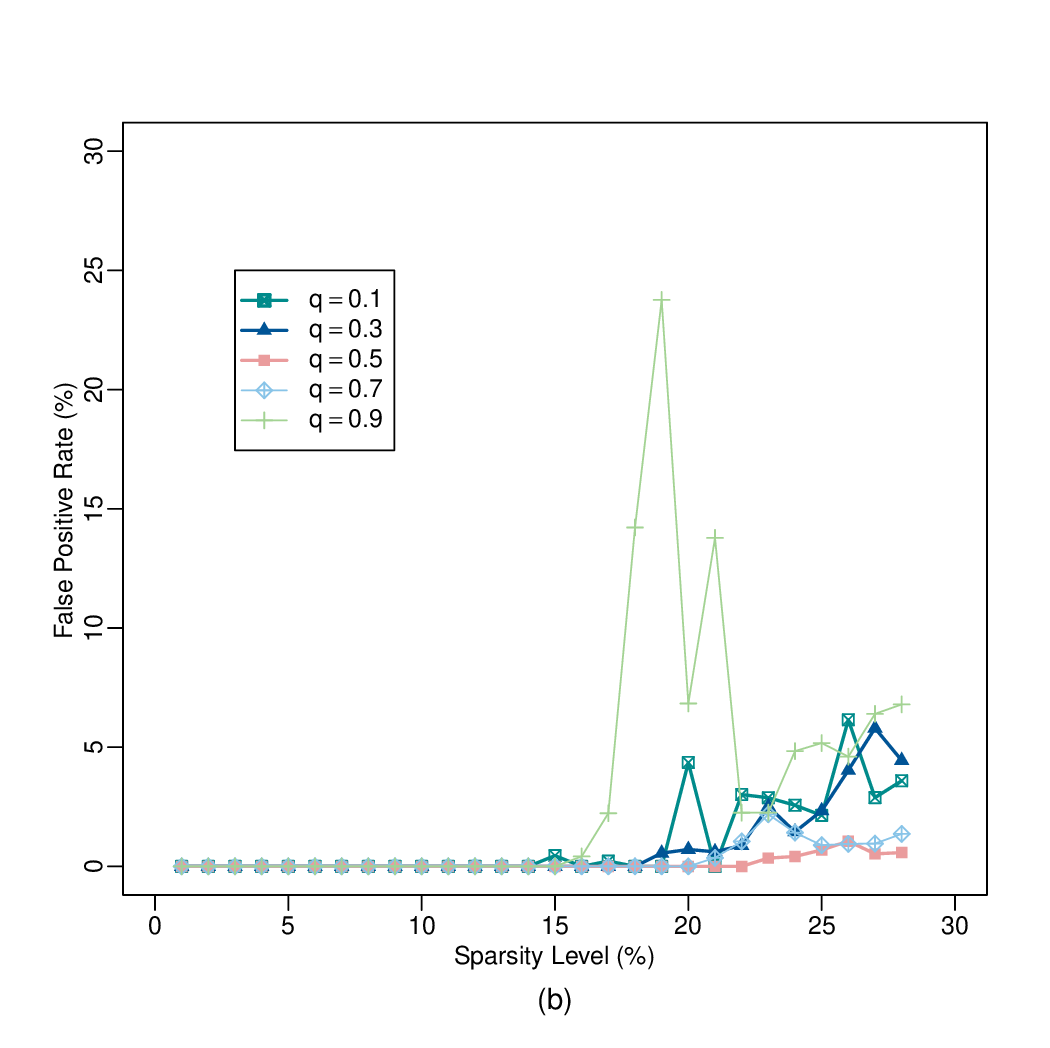}
	\end{minipage}
	
	\caption{The success rates and false positive rates of DS-IRLS-th algorithm with different exponent $q$.}
	\label{fig:choiceq}
\end{figure}

\subsection{Convergence of the DS-IRLS-th algorithm}
\hspace{2em} In this subsection, we analyze the convergence of the proposed DS-IRLS-th algorithm. Two sparsity levels, 1\% and 20\%, are considered in our experiments. \autoref{fig:convergence} presents the objective function values and relative errors across 50 iterations. We observe that, for both sparsity levels, the objective function values and relative errors decrease monotonically and converge rapidly within just a few iterations. This demonstrates the effectiveness and fast convergence of the proposed DS-IRLS-th algorithm.

\begin{figure}[!t]
	\centering
	\begin{minipage}[t]{0.38\linewidth}
		\centering
		\includegraphics[width=\linewidth]{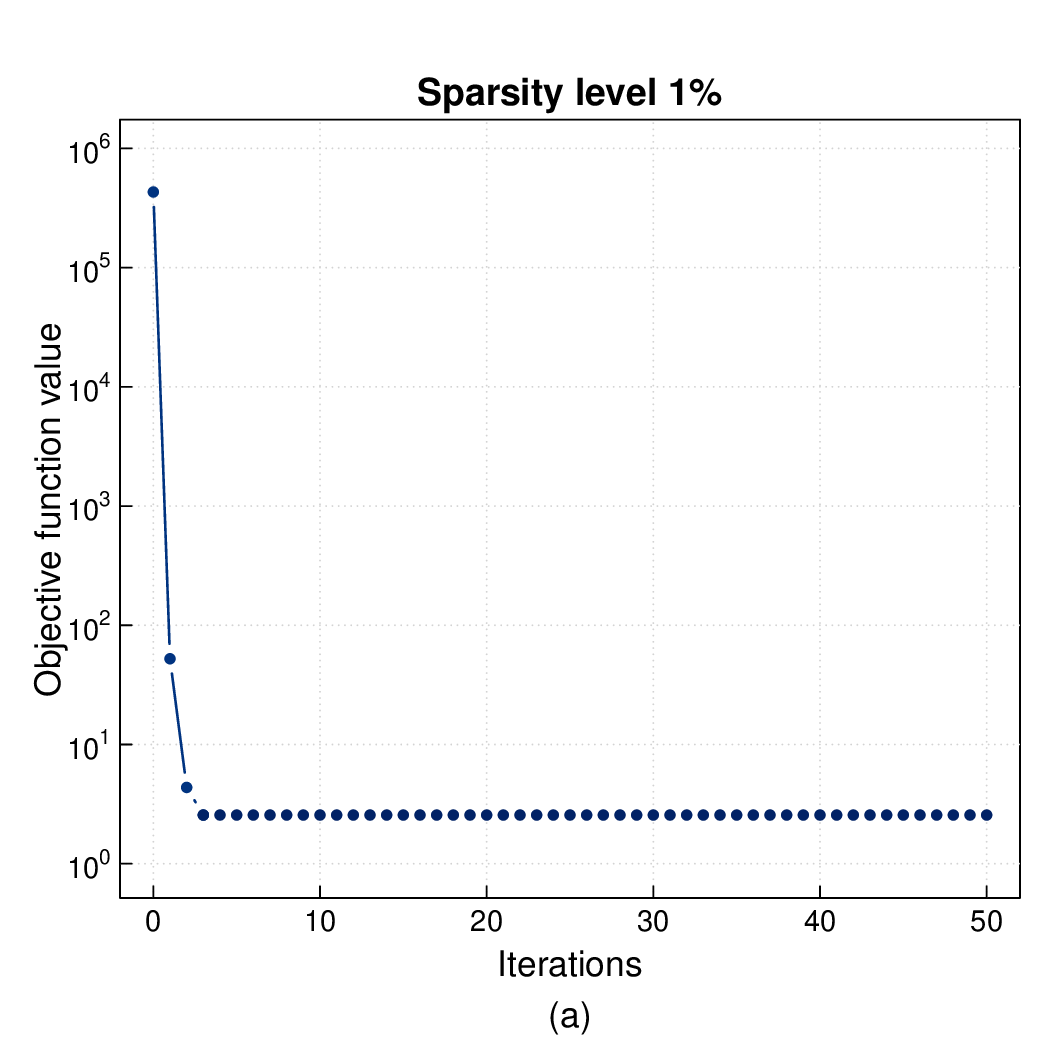} 
	\end{minipage}
	\hspace{1pt} 
	\begin{minipage}[t]{0.38\linewidth}
		\centering
		\includegraphics[width=\linewidth]{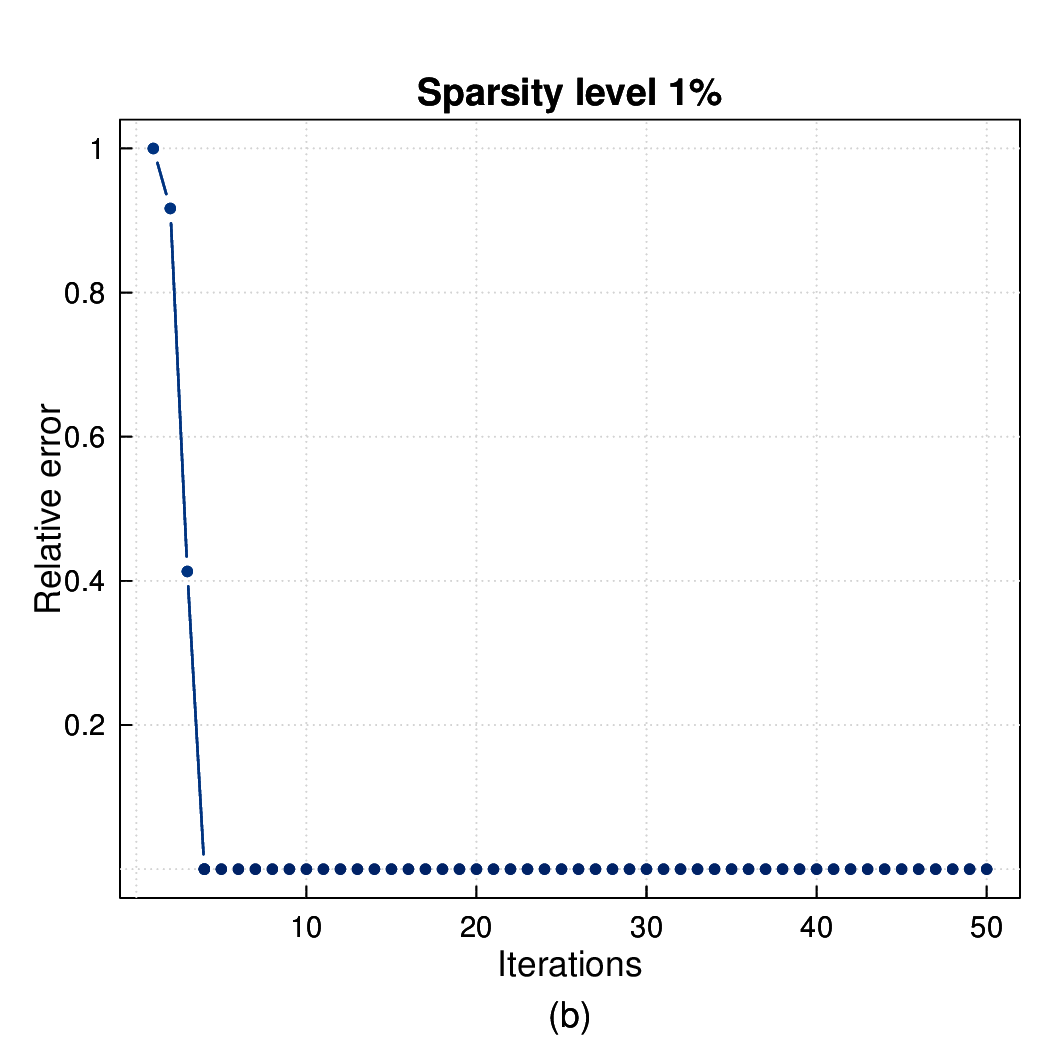} 
	\end{minipage}
	
	\vspace{6pt} 
	\begin{minipage}[t]{0.38\linewidth}
		\centering
		\includegraphics[width=\linewidth]{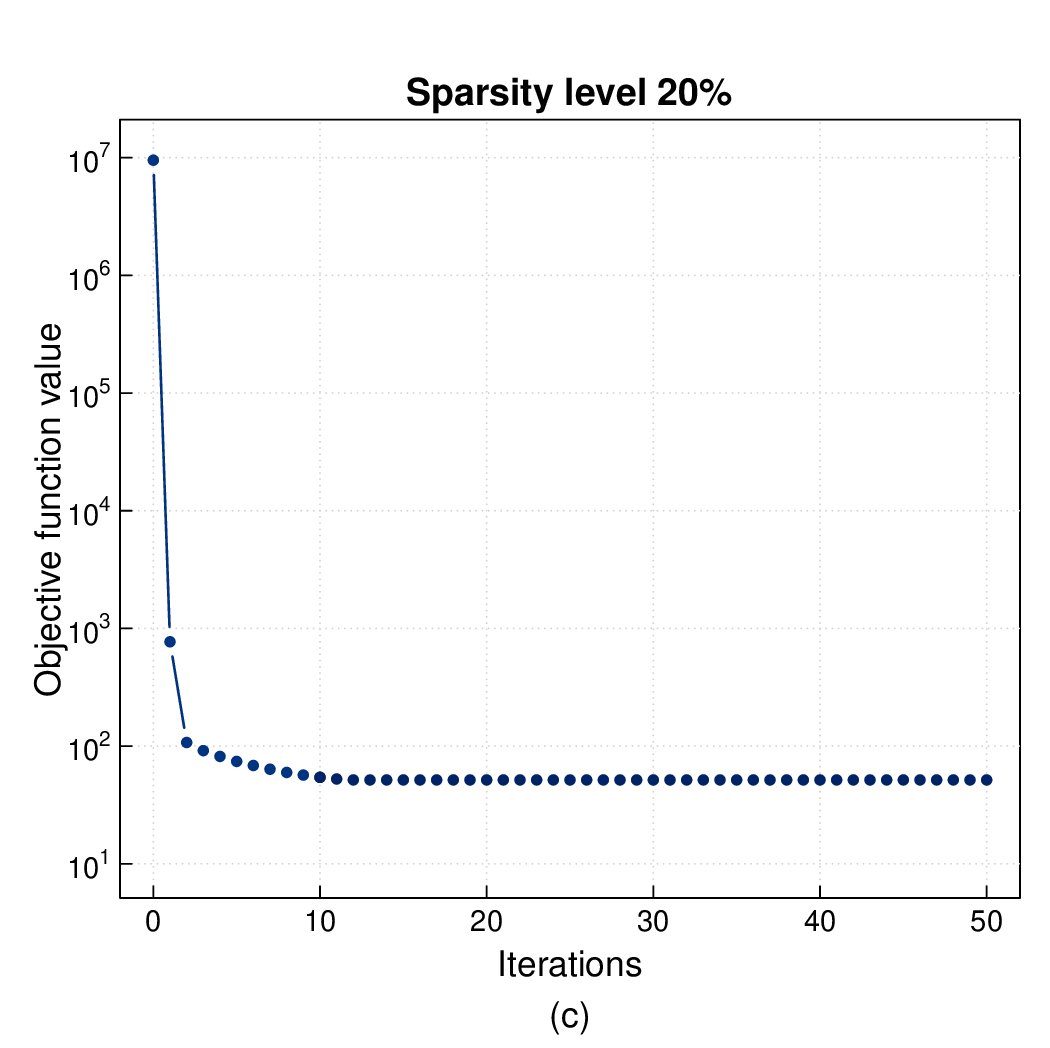} 
	\end{minipage}
	\hspace{1pt}
	\begin{minipage}[t]{0.38\linewidth}
		\centering
		\includegraphics[width=\linewidth]{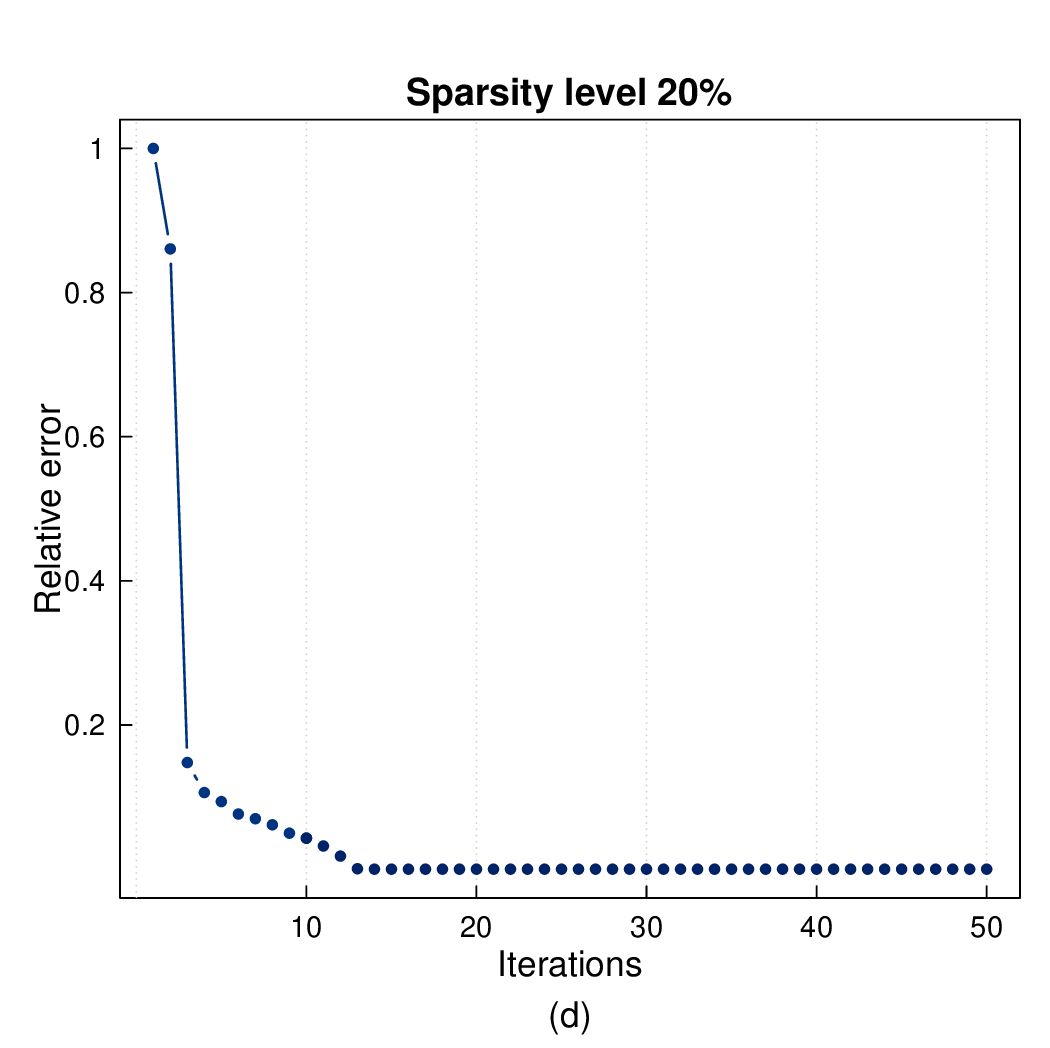} 
	\end{minipage}
	
	\caption{Convergence performance of the DS-IRLS-th algorithm under sparsity levels $1\%$ and $20\%$.}
	\label{fig:convergence}
\end{figure}

\subsection{Comparisons to other methods}
\hspace{2em} We conduct simulation studies to compare the performance of NS$^3$R with several competing approaches in high-dimensional settings, including the smooth Lasso \citep{Hebiri2011}, the spline Lasso \citep{Guo2016}, the square-root fused Lasso \citep{Jiang2021} (SRFL), the group square-root Lasso \citep{Bunea2013} (GSRL), and the structured smooth adjustment for square-root regularization \citep{Xie2020} (SSASR). We use the EBIC criterion from \cite{Chen2008} for parameter tuning to ensure a fair comparison across all models. For methods requiring two tuning parameters, we set $\lambda_2 = \theta\lambda_1$, with the multiplier $\theta$ tuned over the predefined grid $\{2^{-6}, 2^{-5}, \dots, 2^{-1}, 2^{0}\}$. To evaluate prediction accuracy, an additional test set of size $n_{\text{test}} = 10000$ is generated in each simulation, and model performance is assessed using the four metrics defined below:

\begin{itemize}
	\item \textbf{False discovery rate (FDR)}: The proportion of truly null features that are incorrectly selected among the selected features, defined as
	$$\text{FDR} = \frac{|\{i: \beta_i^* = 0, \hat{\beta}_i \neq 0\}|}{|\{i: \hat{\beta}_i \neq 0\}|},$$
	with the denominator set to 1 when no features are selected.
	\item \textbf{False negative rate (FNR)}: The proportion of truly nonzero features that are missed by the selection procedure, defined as
	$$\text{FNR} = \frac{|\{i: \beta_i^* \neq 0, \hat{\beta}_i = 0\}|}{|\{i: \beta_i^* \neq 0\}|}.$$
	\item \textbf{Estimation error (EE)}: The total squared error of the fitted regression entries, defined as
	$$\text{EE} = (\hat{\beta} - \beta^*)^\top (\hat{\beta} - \beta^*).$$
	\item \textbf{Prediction mean squared error (PMSE)}: The average squared prediction error on the independent test set, defined as
	$$\text{PMSE} = \frac{1}{n_{\text{test}}} \sum_{i=1}^{n_{\text{test}}} (y_i - \hat{y}_i)^2,$$
	where $n_{\text{test}}$ is the size of test set, $y_i$ is the true response, and $\hat{y}_i$ is the predicted response. 
\end{itemize}

In all simulations, we fix the sample size $n = 250$, dimension $p = 1000$, and group size 8 (resulting in 125 total groups). Among these groups, exactly 5 are designated as active (i.e., their associated entries are nonzero), with all entries in the remaining groups fixed at zero. The covariates $(X_1, X_2, \dots, X_p)$ are generated from a multivariate normal distribution $\mathcal{N}(0, \Sigma)$, and the response $y$ is a linear function of the covariates with additive zero-mean Gaussian noise. To enable a comprehensive comparison, we consider two signal structures, two correlation structures, and two noise levels across our simulation scenarios, with each setting detailed in the following subsections.

\subsubsection{Piecewise Constant Signals}
\hspace{2em} \textbf{Example 1 (Baseline)} 
We generate piecewise constant signals to emulate the canonical setting where regression entries are homogeneous within each group and heterogeneous across groups. Among the 5 pre-specified active groups, we assign a distinct constant magnitude to all entries within each group. To assess model performance in the weak-signal regime, we set entries in one active group to a low magnitude of 1, and the remaining four active groups to moderate magnitudes of 4, 5, 6, and 7. We then perturb these baseline entries with small zero-mean Gaussian noise of variance 0.01 to mimic real-world measurement noise. For this baseline setting, we fix the regression noise level at $\sigma = 1$ and use the identity correlation matrix $\Sigma = I$.

\textbf{Example 2}
The setup is identical to Example 1, except that the noise level is set to $\sigma = 5$.

\textbf{Example 3}
The setup is identical to Example 1, except that the correlation matrix is set to $\Sigma_{j,k} = 0.5^{|j-k|}$.

\textbf{Example 4}
The setup is identical to Example 3, except that the noise level is set to $\sigma = 5$.

\subsubsection{Complex Smooth Signals}
\hspace{2em} \textbf{Example 5 (Baseline)}
We adopt the same group structure as Example 1, but generate complex smooth signals for active groups (rather than piecewise constant values) to emulate the intricate signal structures common in real-world applications. Each active group is assigned one of five distinct smooth patterns: linear increase, linear decrease, sinusoidal trend, quadratic growth, or exponential decay. We randomly parameterize each active group, adding small zero-mean Gaussian noise of variance 0.01 to mimic real-world measurement noise. For this baseline setting, we fix the regression noise level at $\sigma = 1$ and use the identity correlation matrix $\Sigma = I$.

\textbf{Example 6}
The setup is identical to Example 5, except that the noise level is set to $\sigma = 5$.

\textbf{Example 7}
The setup is identical to Example 5, except that the correlation matrix is set to $\Sigma_{j,k} = 0.5^{|j-k|}$.

\textbf{Example 8}
The setup is identical to Example 7, except that the noise level is set to $\sigma = 5$.

\subsubsection{Numerical Results}
\hspace{2em} The proposed NS$^3$R consistently outperforms all competing methods in variable selection, parameter estimation, and out-of-sample prediction across all simulation scenarios, and shows the highest stability by achieving the smallest standard deviations across replications. All the results for Examples 1–8 are summarized in \autoref{tab:sim_results}, with values averaged over 100 independent replications and standard deviations reported in parentheses. As shown in \autoref{tab:sim_results}, SSASR ranks second in all cases, while smooth Lasso and spline Lasso yield the worst performance in most cases. For Examples 1–4 with piecewise constant signals, SRFL underperforms the two convex group-regularized methods, GSRL and SSASR, across all metrics under low noise levels. When the noise level increases to $\sigma=5$, SRFL achieves comparable estimation and prediction performance to SSASR, owing to its inherent adaptability to piecewise constant structures. This comparable performance comes at the cost of lower variable selection accuracy, as it produces far denser models with substantially higher false discovery rates, and fails to effectively recover the true group-sparse structure.

Comparing Examples 1–2, 5–6 (independent covariates) with Examples 3–4, 7–8 (correlated covariates), the advantage of NS$^3$R is more pronounced for independent covariates. In contrast, all competing methods face much more severe performance degradation relative to the correlated design, while NS$^3$R maintains stable performance throughout. When the noise level increases from $\sigma=1$ to $\sigma=5$, all methods degrade significantly as noise increases, whereas NS$^3$R declines only mildly and maintains a clear advantage over all other approaches. In addition, NS$^3$R consistently achieves exact recovery for low-magnitude signals, whereas all other methods cannot fully recover such signals under high noise levels. NS$^3$R thus provides the best overall performance across all scenarios.

\begin{table}[!htp]
	\centering
	\caption{Simulation results across all examples, evaluating FDR, FNR, EE, and PMSE across 100 replications.}
	\label{tab:sim_results}
	\resizebox{0.95\textwidth}{!}{%
		\begin{tabular}{lcccccc}
			\hline
			& Smooth lasso & Spline lasso & SRFL & GSRL & SSASR & NS$^3$R \\
			\hline
			\multicolumn{7}{c}{Example 1} \\
			FDR & 47.00(0.02) & 47.00(0.03) & 5.54(0.05) & 0.00(0.00) & 0.00(0.00) & 0.00(0.00) \\
			FNR & 1.28(0.02) & 1.45(0.02) & 0.00(0.00) & 0.00(0.00) & 10.20(0.10) & 0.00(0.00) \\
			EE & 30.56(3.14) & 30.71(3.32) & 3.43(1.72) & 2.64(0.92) & 1.51(0.52) & 0.20(0.05) \\
			PMSE & 31.77(3.16) & 31.93(3.34) & 4.44(1.73) & 3.68(0.93) & 2.53(0.52) & 1.20(0.05) \\
			\multicolumn{7}{c}{Example 2} \\
			FDR & 71.26(0.02) & 71.49(0.02) & 2.91(0.04) & 0.17(0.02) & 0.00(0.00) & 0.00(0.00) \\
			FNR & 8.50(0.03) & 8.45(0.03) & 8.15(0.09) & 11.00(0.10) & 10.20(0.10) & 0.00(0.00) \\
			EE & 68.90(8.85) & 67.06(8.45) & 48.61(29.81) & 63.43(22.12) & 47.58(27.73) & 4.70(3.34) \\
			PMSE & 94.00(8.92) & 92.20(8.54) & 73.96(30.05) & 88.76(22.30) & 72.88(27.92) & 29.77(3.29) \\
			\multicolumn{7}{c}{Example 3} \\
			FDR & 14.87(0.05) & 15.02(0.05) & 11.36(0.09) & 0.00(0.00) & 0.00(0.00) & 0.00(0.00) \\
			FNR & 0.00(0.00) & 0.00(0.00) & 0.00(0.00) & 0.00(0.00) & 0.00(0.00) & 0.00(0.00) \\
			EE & 1.34(0.20) & 1.36(0.21) & 1.10(0.19) & 0.60(0.11) & 0.43(0.08) & 0.30(0.08) \\
			PMSE & 2.62(0.19) & 2.64(0.19) & 2.07(0.31) & 1.73(0.13) & 1.57(0.11) & 1.19(0.05) \\
			\multicolumn{7}{c}{Example 4} \\
			FDR & 15.35(0.05) & 15.34(0.05) & 13.93(0.08) & 0.00(0.00) & 0.00(0.00) & 0.00(0.00) \\
			FNR & 0.95(0.01) & 0.93(0.01) & 0.00(0.00) & 0.00(0.00) & 0.00(0.00) & 0.00(0.00) \\
			EE & 25.76(5.94) & 25.12(5.44) & 9.01(2.01) & 10.10(2.01) & 8.38(1.70) & 6.40(1.85) \\
			PMSE & 58.45(7.78) & 57.97(7.26) & 41.61(3.95) & 40.29(3.25) & 39.21(2.86) & 29.35(1.20) \\
			\multicolumn{7}{c}{Example 5} \\
			FDR & 24.95(0.03) & 25.30(0.03) & 4.54(0.05) & 0.00(0.00) & 0.00(0.00) & 0.00(0.00) \\
			FNR & 0.68(0.01) & 0.58(0.01) & 0.63(0.01) & 0.00(0.00) & 0.00(0.00) & 0.00(0.00) \\
			EE & 14.98(0.61) & 14.89(0.61) & 2.71(1.57) & 1.15(0.27) & 0.92(0.26) & 0.19(0.05) \\
			PMSE & 15.78(0.60) & 15.68(0.60) & 3.76(1.59) & 2.15(0.26) & 1.92(0.25) & 1.20(0.05) \\
			\multicolumn{7}{c}{Example 6} \\
			FDR & 53.39(0.08) & 53.25(0.08) & 2.69(0.03) & 0.33(0.02) & 0.00(0.00) & 0.00(0.00) \\
			FNR & 6.30(0.02) & 6.05(0.02) & 4.33(0.02) & 0.00(0.00) & 0.00(0.00) & 0.00(0.00) \\
			EE & 57.75(14.59) & 55.40(12.93) & 24.83(7.48) & 22.58(5.61) & 15.43(4.34) & 4.58(1.74) \\
			PMSE & 82.01(14.53) & 79.63(12.83) & 50.20(7.64) & 47.49(5.59) & 40.47(4.32) & 29.60(1.77) \\
			\multicolumn{7}{c}{Example 7} \\
			FDR & 19.60(0.07) & 19.79(0.06) & 9.48(0.04) & 0.00(0.00) & 0.00(0.00) & 0.00(0.00) \\
			FNR & 0.83(0.01) & 0.73(0.01) & 2.50(0.00) & 0.00(0.00) & 0.00(0.00) & 0.00(0.00) \\
			EE & 1.94(0.39) & 1.87(0.40) & 1.72(0.50) & 0.76(0.16) & 0.49(0.13) & 0.28(0.08) \\
			PMSE & 3.11(0.48) & 3.04(0.46) & 3.05(0.93) & 1.93(0.19) & 1.65(0.18) & 1.18(0.05) \\
			\multicolumn{7}{c}{Example 8} \\
			FDR & 18.57(0.07) & 18.16(0.06) & 10.25(0.05) & 0.00(0.00) & 0.00(0.00) & 0.00(0.00) \\
			FNR & 4.80(0.02) & 4.78(0.02) & 3.50(0.02) & 0.00(0.00) & 0.00(0.00) & 0.00(0.00) \\
			EE & 27.63(6.07) & 26.67(6.07) & 18.29(7.63) & 10.75(1.86) & 7.19(1.73) & 5.68(1.44) \\
			PMSE & 59.89(7.86) & 59.30(7.85) & 59.45(16.60) & 41.62(3.27) & 37.62(3.39) & 29.14(1.11) \\
			\hline
		\end{tabular}
	}%
\end{table}

\section{An application to gene expression microarray analysis}\label{sec:application}
\hspace{2em} We apply our analytical framework to a gene expression dataset previously studied in \citep{Reppe2010}. This dataset has been deposited in the European Bioinformatics Institute (EMBL-EBI) Array-Express repository, ID: E-MEXP-1618. It contains expression microarray profiles of 22815 probe sets and bone biopsies from 84 women, with raw data normalized via the RMA method as detailed in \citep{Reppe2010}. For our analysis, bone mineral density (BMD) is the outcome of interest, quantified as hip T-scores from the participants and standardized to have zero mean and unit standard deviation. Gene expression levels are analyzed on a log2 scale. To incorporate domain-specific prior knowledge of bone metabolism into the model, we group all genes according to two sources of biological evidence. These include functionally validated bone metabolism-related genes from published studies \citep{Tharmaratnam2016, Zeng2021}, and well-characterized bone metabolic pathways including bone resorption, osteoblast differentiation and Wnt signaling cascades. We discard only the group of unannotated genes while retaining all other biologically informed grouping information, and sort all genes by their marginal importance relative to the bone density outcome. This strategy integrates prior biological insights into the leaning process, and improves the biological interpretability of the feature sets identified by our penalized regression model.

To illustrate the advantage of incorporating the structured smoothness prior, we compare our proposed NS$ $R approach with two competing methods. These are the square-root fused Lasso (srfl) \citep{Jiang2021} and the external information tuned (xtune) Lasso \citep{Zeng2021}. For xtune Lasso, we use the \textit{xtune} and \textit{glmnet} \citep{Friedman2010} packages in R, and construct the external information matrix $Z$ based on the predefined gene grouping structure. Each column of $Z$ is a binary indicator for a specific group, with entries set to 1 for genes in the corresponding group and 0 otherwise. Regularization parameters for all methods are selected via the SCV-BIC criterion \citep{She2012} to ensure consistency and fairness of the comparison. We randomly split the full dataset into a training set containing 80\% of the 84 total observations (68 samples) and a held-out test set with the remaining 20\% (18 samples). We fit all methods on the training data, and evaluate their predictive performance on the test set. This random split procedure is repeated for 100 independent simulations. \autoref{tab:gene_selection} reports the mean number of selected genes and mean PMSE, with corresponding standard deviations in parentheses. Across all simulations, NS$^3$R consistently selects a substantially smaller feature set than both xtune Lasso and srfl, with a markedly lower standard deviation that demonstrates superior stability in feature selection. Correspondingly, NS$^3$R achieves the lowest mean PMSE and outperforms both competing methods. This indicates that the sparse, stable feature sets identified by NS$^3$R yield better predictive performance, while the larger, inconsistent gene sets selected by xtune Lasso and srfl f ail to deliver comparable improvements in prediction accuracy.

Notably, genes frequently selected by NS$^3$R with a selection frequency of $\geq 30\%$ across simulations are well-validated regulators of bone metabolism, and include core components of the Wnt signaling pathway such as DKK1, SOST, and WIF1. These core genes are robustly identified across all 100 simulations, confirming that NS$^3$R can stably and reproducibly identify biologically meaningful signals relevant to bone metabolism. In comparison, xtune Lasso and srfl select a larger number of genes with lower selection frequencies. Most of the low-frequency genes selected by these competing methods have no validated functional link to bone metabolism, showing that these methods are more likely to capture false-positive features in high-dimensional data.

In summary, these results demonstrate that NS$^3$R achieves a favorable balance among sparse feature selection, predictive performance, and biological interpretability in high-dimensional genomic data analysis. By incorporating a structured smoothness prior adapted to the inherent grouping structure of genes, NS$^3$R provides a robust and biologically meaningful approach for feature selection in high-dimensional gene expression studies.

\begin{table}[!htp]
	\centering
	\caption{The frequently selected genes, mean number of selected genes and mean PMSE across 100 replications on the test data.}
	\label{tab:gene_selection}
	\resizebox{0.95\textwidth}{!}{%
		\begin{tabularx}{0.95\textwidth}{@{}p{30pt}Xp{65pt}p{60pt}@{}}
			\toprule
			Method & Frequently selected genes ($\geq 30\%$ in 100 simulations) & Number of selected genes & PMSE \\
			\midrule
			srfl & ACSL3 (100\%), DKK1 (98\%), C1orf61 (78\%), IKBKG (70\%), GATA1 (69\%), HSD17B4 (69\%), RRAS2 (68\%), CTNNBIP1 (67\%), FBN2 (67\%), GNAS (66\%), MIR18A (66\%), WIF1 (57\%), WWOX (57\%), SOST (54\%), IFITM1 (47\%), SLC38A10 (47\%), IL17A (43\%), PTHLH (42\%), CDH11 (39\%), CARTPT (35\%), SLC10A7 (34\%), SNRNP200 (34\%), TPH1 (34\%), NFATC2 (33\%), SMAD4 (32\%), HIF1A (31\%) & 21.75 (10.30) & 0.75 (0.27) \\
			\midrule
			xtune & ACSL3 (98\%), DKK1 (96\%), WIF1 (96\%), ELTD1 (93\%), C1orf61 (88\%), SOST (62\%), SOX4 (58\%), KDM4A (56\%), PDGFA (32\%) & 17.89 (43.63) & 0.68 (0.27) \\
			\midrule
			NS$^3$R & ACSL3 (100\%), DKK1 (100\%), ELTD1 (100\%), SFRP1 (100\%), SFRP4 (100\%), SOST (100\%), WIF1 (100\%) & 8.93 (2.41) & 0.65 (0.24) \\
			\bottomrule
		\end{tabularx}
	}%
\end{table}

\section{Conclusions and future works}\label{sec:conclusion}
\hspace{2em} In this work, we develop the nonconvex structured smooth adjustment for square-root regularization framework (NS$^3$R) for high-dimensional structured sparse learning. It couples a joint nonconvex $\ell_{2,q}$ penalty with a $\sigma$-free square-root loss to simultaneously encode group sparsity and structured smoothness priors. This design mitigates the inherent over-shrinkage estimation bias from convex regularizers, and the dependency of tuning parameter selection on the unknown noise level $\sigma$, which greatly facilitates the work of parameter tuning. We establish rigorous non-asymptotic estimation and prediction error bounds for the NS$^3$R estimator under a mild regularity condition, providing reliable theoretical guarantees for its application in $p \gg n$ settings. For implementation, we develop an efficient double-screened IRLS with thresholding (DS-IRLS-th) algorithm for the resulting nonconvex, nonsmooth and non-Lipschitz optimization problem, with a rigorous proof of its global convergence under mild initial value conditions. Its nested double-screening design implements a support-shrinking strategy, which guarantees a rapid reduction of computational complexity across iterations. Numerical experiments show that the algorithm converges within few iterations, with a runtime below 1 second for moderate-sized datasets. Comprehensive experiments on synthetic and real-world gene expression datasets confirm that NS$^3$R delivers consistently better finite-sample performance than state-of-the-art regularization methods.

In our future work, we will extend the proposed framework to accommodate more flexible structural priors and complex data settings. We will generalize the structured sparsity-inducing penalties to enable bi-level selection, enhancing its robustness against mis-specified grouping structures. We will also explore a broader class of nonconvex penalties, including group SCAD, group MCP and capped folded concave penalties, to achieve more accurate signal recovery in high-dimensional regimes. Further, we will extend the framework to graph-structured data to model label dependency across nodes via oriented incidence matrices, beyond simple piecewise smoothness assumptions. These extensions will provide a unified tool for high-dimensional data modeling in bioinformatics, signal processing, financial econometrics, image processing and other related application areas.
\section*{Acknowledgments}
This work is supported by the National Natural Science Foundation of China [Grant No. 12401345].

\appendix 

\renewcommand{\thesection}{\Alph{section}}
\titleformat{\section}{\bfseries\fontsize{14}{16}\selectfont}{Appendix \thesection}{1em}{}

\renewcommand{\thelem}{\thesection.\arabic{lem}}
\renewcommand{\thedefinition}{\thesection.\arabic{definition}}
\renewcommand{\thepro}{\thesection.\arabic{pro}}
\renewcommand{\thermk}{\thesection.\arabic{rmk}}

\setcounter{section}{0} 
\setcounter{definition}{0} 
\setcounter{lem}{0} 
\setcounter{pro}{0} 
\setcounter{rmk}{0} 
\setcounter{equation}{0} 

\section{Preliminary Definitions and Supporting Lemmas}\label{app:prelim_defs}

\subsection{Subdifferential}\label{app:subdiff}
\hspace{2em} In this section, we recall the basic definition of the subdifferential.

\begin{definition}[Subdifferentials, \citep{Rockafellar1998}]
	\label{def:subdiff}
	Let $\phi: \mathbb{R}^p \to (-\infty, +\infty]$ be a proper and lower semicontinuous function.
	\begin{enumerate}[label=(\roman*), leftmargin=1.5em, topsep=0pt, partopsep=0pt, align=left]
		\item The regular subdifferential of $\phi$ at $x \in \text{dom}\,\phi = \{x \in \mathbb{R}^p: \phi(x) < +\infty\}$ is defined as
		\[
		\hat{\partial}\phi(x) := \left\{ u \in \mathbb{R}^p : \lim_{\substack{y \to x \\ y \neq x}} \inf \frac{\phi(y) - \phi(x) - \langle u, y - x \rangle}{\|y - x\|_2} \geq 0 \right\};
		\]
		\item The (limiting) subdifferential of $\phi$ at $x \in \text{dom}\,\phi$ is defined as
		\[
		\partial\phi(x) := \left\{ u \in \mathbb{R}^p : \exists x^k \to x,\ \phi(x^k) \to \phi(x) \text{ and } \hat{\partial}\phi(x^k) \ni u^k \to u,\ \text{as } k \to \infty \right\}.
		\]
	\end{enumerate}
\end{definition}
\subsection{KL Property and Semi-Algebraic Functions}
\label{app:kl_property} 

\hspace{2em} The Kurdyka-\L{}ojasiewicz (KL) property plays a fundamental role in the convergence analysis of many optimization algorithms \citep{Attouch2010, Bolte2014}. A broad class of functions, including semi-algebraic, subanalytic, and log-exp functions, are KL functions. The definition of the KL property is given below.

\begin{definition}[KL property, \cite{Attouch2010}]
	\label{def:kl_property}
	The function $\phi: \mathbb{R}^p \to (-\infty, +\infty]$ is said to have the Kurdyka-\L{}ojasiewicz (KL) property at $\bar{x} \in \text{dom}\,\partial\phi := \{x \in \mathbb{R}^p: \partial\phi(x) \neq \emptyset\}$ if there exists $\eta \in (0, +\infty]$, a neighborhood $U$ of $\bar{x}$, and a continuous concave function $\psi: [0, \eta) \to \mathbb{R}_+$ such that
	\begin{enumerate}[label=(\roman*), leftmargin=1.5em, topsep=0pt, partopsep=0pt, align=left]
		\item[(i)] $\psi(0) = 0$;
		\item[(ii)] $\psi$ is continuously differentiable on $(0, \eta)$;
		\item[(iii)] for all $t \in (0, \eta)$, $\psi'(t) > 0$;
		\item[(iv)] for all $x \in U \cap \{\phi(\bar{x}) < \phi(x) < \phi(\bar{x}) + \eta\}$, the KL inequality holds:
		\[
		\psi'\big(\phi(x) - \phi(\bar{x})\big)\operatorname{dist}\big(0, \partial\phi(x)\big) \geq 1.
		\]
	\end{enumerate}
If $\phi$ satisfies the KL property at each point of $\text{dom}\,\partial\phi$, $\phi$ is called a KL function.
\end{definition}

Semi-algebraic functions form an important subclass of KL functions, and their definition is given as follows.
\begin{definition}[Semi-algebraic sets and functions, \citep{Bolte2014}]
	\label{def:semi_algebraic}
	\leavevmode\\\vspace{-\baselineskip} 
	\begin{enumerate}[label=(\roman*), leftmargin=1.5em, topsep=0pt, partopsep=0pt, align=left]
		\item A subset $\mathcal{H}$ of $\mathbb{R}^p$ is a real semi-algebraic set if there exists a finite number of real polynomial functions $g_{ij}, h_{ij}: \mathbb{R}^p \to \mathbb{R}$ such that
		\[
		\mathcal{H} = \bigcup_j \bigcap_i \left\{ u \in \mathbb{R}^p : g_{ij}(u) = 0 \text{ and } h_{ij}(u) < 0 \right\}.
		\]
		\item A function $\phi: \mathbb{R}^p \to (-\infty, +\infty]$ is called semi-algebraic if its graph
		\[
		\left\{ (u, t) \in \mathbb{R}^{p+1} : \phi(u) = t \right\}
		\]
	\end{enumerate}
is a semi-algebraic subset of $\mathbb{R}^{p+1}$.
\end{definition}

The objective function $F(\beta)$ defined in \eqref{eq:opt17} is semi-algebraic. Indeed, the square-root loss $\|y - X\beta\|_2$ is semi-algebraic (see \cite[Example 2]{Bolte2014}), and for a rational exponent $q \in (0,1)$, the functions $\|\beta_j\|_2^q$ and $\|D_j \beta_j\|_2^q$ are semi-algebraic. Since finite sums and compositions of semi-algebraic functions are semi-algebraic, $F$ is semi-algebraic and thus a KL function (see \cite[Theorem 3]{Bolte2014}). This ensures that the KL property holds at every point of interest in our convergence analysis.

\subsection{Supporting Preliminary Lemmas}
\label{app:supporting_lemmas} 
\hspace{2em} In order to prove our main results, it is necessary to present some technical lemmas. The following two lemmas are the key inequalities used in the proof.

\begin{lem}
	\label{lem:key_inequalities}
	Let $x, y \in \mathbb{R}^p$. Then the following relations hold:
	\begin{enumerate}[
		label=(\roman*), 
		leftmargin=2.5em, 
		labelwidth=1.2em, 
		labelsep=0.5em, 
		align=left, 
		topsep=0pt, 
		partopsep=0pt, 
		itemsep=0.8em,
		parsep=0pt
		]
		\item $\|x\|_{q_2} \leq \|x\|_{q_1} \leq p^{\frac{1}{q_1} - \frac{1}{q_2}} \|x\|_{q_2}, \text{ for any } 0 < q_1 \leq q_2 < +\infty.$ \hfill \textup{(A.1)}
		\label{eq:A.1}
		
		\item $\|x\|_{2,q}^q - \|y\|_{2,q}^q \leq \|x - y\|_{2,q}^q, \text{ for any } 0 < q \leq 1.$ \hfill \textup{(A.2)}
		\label{eq:A.2}
		
		\item Given $a > 0$, if $\|x\|_2 + \|y\|_2 \leq M$, then
		\[
		\|x\|_2^q + \|y\|_2^q \leq 2^{1 - q} M^q, \text{ for any } 0 < q \leq 1.
		\tag{A.3}
		\label{eq:A.3}
		\]
	\end{enumerate}
\end{lem}
\begin{proof}
	The left-hand side of (i) follows from the monotonicity of the $\ell_q$ norms. For the right-hand side of (i), let $r = q_2/q_1 \geq 1$ and let $r' = r/(r - 1)$ be its H\"older conjugate. Applying H\"older's inequality to the sequences $(|x_i|^{q_1})_{i=1}^p$ and the constant sequence $(1)_{i=1}^p$ gives
	\[
	\sum_{i=1}^p |x_i|^{q_1} \leq \left(\sum_{i=1}^p |x_i|^{q_1 r}\right)^{1/r} \left(\sum_{i=1}^p 1^{r'}\right)^{1/r'} = \left(\sum_{i=1}^p |x_i|^{q_2}\right)^{q_1/q_2} p^{1 - q_1/q_2}. 
	\]
	Raising both sides to the power $1/q_1$ yields
	\[
	\|x\|_{q_1} = \left(\sum_{i=1}^p |x_i|^{q_1}\right)^{1/q_1} \leq \left(\sum_{i=1}^p |x_i|^{q_2}\right)^{1/q_2} p^{\frac{1}{q_1} - \frac{1}{q_2}} = p^{\frac{1}{q_1} - \frac{1}{q_2}} \|x\|_{q_2}.
	\tag{A.4}
	\label[equation]{A.4}
	\]
	
	The second inequality is a direct application of Lemma~2 in \citep{Hu2017}. For the third inequality, when $q = 1$, the result follows immediately from the triangle inequality. When $0 < q < 1$, let $u = \|x\|_2$ and $v = \|y\|_2$. Then $u, v \geq 0$ and $u + v \leq M$. Since the function $t \mapsto t^q$ is concave on $[0, \infty)$, by Jensen's inequality we have
	\[
	\left(\frac{u + v}{2}\right)^q \geq \frac{u^q + v^q}{2}. 
	\tag{A.5}
	\label[equation]{A.5}
	\]
	Rearranging this inequality yields
	\[
	u^q + v^q \leq 2^{1 - q} (u + v)^q.
	\]
	Since $u + v \leq M$ and $s \mapsto s^q$ is strictly increasing on $(0, \infty)$, we conclude
	\[
	\|x\|_2^q + \|y\|_2^q = u^q + v^q \leq 2^{1 - q} M^q.
	\tag{A.6}
	\label[equation]{A.6}
	\]
This completes the proof.
\end{proof}

\begin{lem}
	\label{lem:strong_convexity_ineq}
	Given $0 < q < 1$, for all $x, y \in \mathbb{R}^p$ with $x \neq 0$,
	\[
	\|x\|_2^q - \|y\|_2^q - q\|x\|_2^{q-2} (x - y)^\top y \geq \frac{q}{2} \|x\|_2^{q-2} \|x - y\|_2^2.
	\tag{A.7}
	\label{eq:A.7}
	\]
\end{lem}
This lemma comes from Lemma 6.3 in \citep{Feng2020}.
\begin{lem}
	\label{eventA}
For a given $\delta_1 \in (0,1)$, let $t = \sqrt{\frac{4\ln(1/\delta_1)}{n}} + \frac{4\ln(1/\delta_1)}{n}$ and define
\[
\mathscr{A} := \left\{ \frac{\|\varepsilon\|_2}{\sqrt{n}} \leq \sigma\sqrt{1 + t} \right\}.
\tag{A.8}
\label{eq:A.8}
\]

Then,
\[
\mathbb{P}(\mathscr{A}) \geq 1 - \delta_1.
\tag{A.9}
\label{eq:A.9}
\]
\end{lem}

The proof of this result is a direct application of Lemma 8.1 in \citep{Buhlmann2011}.

\begin{lem}
	\label{eventB}
Assume that $T_j < n$ for all $1 \leq j \leq J$. Let $\zeta = \max_j \zeta_j$, where $\left\{ \zeta_j = \|X_j\|_2^2 / n \right\}$, and $\delta_2 \in (0,1)$ be given such that $16 \log\left( \frac{2J}{\delta_2} \right) \leq n - T_{\max}$ and $s^{2/q} \log\left( \frac{2J}{\delta_2} \right) \geq \frac{\sigma^2(1+t)(n-T_{\max})T_{\min}}{T_{\max}\zeta\rho^2}$. If
\[
\lambda_1 \geq \frac{n}{\sqrt{n-T_{\max}}} T_{\max}^{\frac{1-q}{2}} \sqrt{2\zeta} s^{\frac{1-q}{q}} \rho^{1-q} \bar{\gamma} \left( 1 + \sqrt{\frac{2\log(2J/\delta_2)}{T_{\min}}} \right),
\tag{A.10}
\label{eq:A.10}
\]
we have $\mathbb{P}(\mathscr{B}) \geq 1 - \delta_2$.
\end{lem}
\begin{proof}
We first observe that
\[
\begin{aligned}
	\mathbb{P}(\mathscr{B}^c) &= \mathbb{P}\left( (2\bar{\rho})^{1-q} V > \frac{\lambda_1}{\bar{\gamma}} \right)\\
	&\leq \mathbb{P}\left( (2\bar{\rho})^{1-q} V > \frac{n}{\sqrt{n-T_{\max}}} T_{\max}^{\frac{1-q}{2}} \sqrt{2\zeta} s^{\frac{1-q}{q}} \rho^{1-q} \left( 1 + \sqrt{\frac{2\log(2J/\delta_2)}{T_{\min}}} \right) \right) \\
	&\leq \mathbb{P}\left( \left( \frac{\sigma n \sqrt{1+t}}{2^{-q} \lambda_1 s T_{\max}^{q/2} \rho^q} + \frac{1}{2} \right)^{\frac{1-q}{q}} V > \frac{n}{\sqrt{n-T_{\max}}} \sqrt{2\zeta} \left( 1 + \sqrt{\frac{2\log(2J/\delta_2)}{T_{\min}}} \right) \right)
\end{aligned}
\]

It follows from \eqref{eq:A.10} that
\[
\begin{aligned}
	\frac{\sigma n \sqrt{1+t}}{2^{-q} \lambda_1 s T_{\max}^{q/2} \rho^q}
	&\leq \frac{\sigma n \sqrt{1+t}}{2^{-q} \frac{n}{\sqrt{n-T_{\max}}}\sqrt{2\zeta} T_{\max}^{1/2} s^{1/q} \rho \bar{\gamma} \left( 1 + \sqrt{\frac{2\log(2J/\delta_2)}{T_{\min}}} \right)} \\
	&\leq \frac{\sigma n \sqrt{1+t}}{2^{-q} \frac{n}{\sqrt{n-T_{\max}}} \sqrt{2\zeta} T_{\max}^{1/2} s^{1/q} \rho \bar{\gamma} \sqrt{\frac{2\log(2J/\delta_2)}{T_{\min}}}} \leq \frac{1}{2}
\end{aligned}
\]

By Lemma II.1 in \citep{Bunea2013}, we can obtain that
\[
\mathbb{P}(\mathscr{B}^c) \leq \mathbb{P}\left( V > \frac{n}{\sqrt{n-T_{\max}}} \sqrt{2\zeta} \left( 1 + \sqrt{\frac{2\log(2J/\delta_2)}{T_{\min}}} \right) \right) \leq \delta_2
\tag{A.11}
\label{eq:A.11}
\]
This concludes the proof.
\end{proof}

\begin{lem}
	\label{lem:group_support_converge}
	In the proposed DS-IRLS-th algorithm, the group support set sequence $\{\mathbb{G}^k\}$ and the difference active set sequence $\{\mathbb{A}^k\}$ satisfies $\mathbb{G}^{k+1} \subseteq \mathbb{G}^k$ and $\mathbb{A}^{k+1} \subseteq \mathbb{A}^k$ for all $k \geq 0$, and thus both sequences converge. Suppose that the total number of groups $J$ is finite, then there exists $K \geq 0$ and $\mathbb{G}^*, \mathbb{A}^* \subseteq \{1, \dots, J\}$ such that for all $k \geq K$,
	\[
	\mathbb{G}^k = \mathbb{G}^*, \quad \mathbb{A}^k = \mathbb{A}^*.
	\tag{A.12}
	\label{eq:A.12}
	\]
\end{lem}

\begin{proof}
	Both $\{\mathbb{G}^k\}$ and $\{\mathbb{A}^k\}$ are non-increasing sequences (i.e., $\mathbb{G}^{k+1} \subseteq \mathbb{G}^k$, $\mathbb{A}^{k+1} \subseteq \mathbb{A}^k$), and the total number of groups $J$ is finite, thus they must converge to fixed sets within finite steps.
\end{proof}
\section{Technical Proofs of Main Results}
\label{app:technical_proofs} 
\noindent\textbf{\emph{Proof of \Cref{pro1}}}\quad By \Cref{lem:key_inequalities} (iii) and \eqref{eq:opt7}, one has that
\[
\sum_{j\in \mathbb{G}} T_j^{q/2} \left[ \|\beta_j^*\|_2^2 + \frac{1-\alpha}{\alpha} \|D_j\beta_j^*\|_2^2 \right] 
\leq T_{\max}^{q/2} s 2^{1-q} \bigl(2^{-2/q}\rho\bigr)^q 
\leq 2^{-(q+1)} s T_{\max}^{q/2} \rho^q 
\tag{B.1}
\label{eq:B.1}
\]

Since $\hat{\beta}$ is a global solution of criterion \eqref{eq:opt4}, we have
\[
\frac{\lambda\alpha}{n} \sum_{j=1}^J T_j^{q/2} \|\hat{\beta}_j\|_2^2 
\leq \frac{1}{\sqrt{n}} \|y - X\beta^*\|_2 
+ \frac{\lambda\alpha}{n} \sum_{j\in \mathbb{G}} T_j^{q/2} \|\beta_j^*\|_2^2 
+ \frac{\lambda(1-\alpha)}{n} \sum_{j\in \mathbb{G}} T_j^{q/2} \|D_j\beta_j^*\|_2^2
\]

Since on $\mathscr{A}$ we have $\|\varepsilon\|_2/\sqrt{n} \leq \sigma\sqrt{1+t}$, then by \eqref{eq:B.1}, it follows that
\[
\begin{aligned}
	\sum_{j=1}^J T_j^{q/2} \|\hat{\beta}_j\|_2^2 
	&\leq \frac{n}{\sqrt{n}\lambda\alpha} \|y - X\beta^*\|_2 
	+ \sum_{j\in \mathbb{G}} T_j^{q/2} \left[ \|\beta_j^*\|_2^2 + \frac{1-\alpha}{\alpha} \|D_j\beta_j^*\|_2^2 \right] \\
	&\leq \frac{n}{\sqrt{n}\lambda\alpha} \|\varepsilon\|_2 
	+ 2^{-q} s T_{\max}^{q/2} \rho^q / 2 \leq \bar{\rho}^q
\end{aligned}
\]
where 
\[
\bar{\rho} = \left( \sigma n\sqrt{1+t}/\lambda\alpha + 2^{-q} s T_{\max}^{q/2} \rho^q / 2 \right)^{1/q}.
\]

Note that $0 < q < 1$ and $T_j \geq 1$, then it follows from \Cref{lem:key_inequalities} (i) that
\[
\begin{aligned}
	\sum_{j=1}^J \sqrt{T_j} \|\hat{\beta}_j - \beta_j^*\|_2 
	&\leq \sum_{j=1}^J \sqrt{T_j} \left[ \|\hat{\beta}_j\|_2 + \|\beta_j^*\|_2 \right] \leq \sum_{j=1}^J \sqrt{T_j} \|\hat{\beta}_j\|_2 
	+ \sum_{j\in \mathbb{G}} \sqrt{T_j} \|\beta_j^*\|_2 \\
	&\leq \left[ \sum_{j=1}^J T_j^{q/2} \|\hat{\beta}_j\|_2^q \right]^{1/q} 
	+ \left[ \sum_{j\in \mathbb{G}} T_j^{q/2} \|\beta_j^*\|_2^q \right]^{1/q} \\
	&\leq \bar{\rho} + \left( 2^{-q} s T_{\max}^{q/2} \rho^q / 2 \right)^{1/q} 
	\leq 2\bar{\rho}
\end{aligned}
\]

Therefore,
\[
\begin{aligned}
	\sum_{j=1}^J \sqrt{T_j} \|\hat{\beta}_j - \beta_j^*\|_2 
	&\leq \sum_{j=1}^J T_j^{(1-q)/2} \|\hat{\beta}_j - \beta_j^*\|_2^{1-q} \cdot T_j^{q/2} \|\hat{\beta}_j - \beta_j^*\|_2^q \\
	&\leq \max_{1\leq j\leq J} T_j^{(1-q)/2} \|\hat{\beta}_j - \beta_j^*\|_2^{1-q} \cdot \sum_{j=1}^J T_j^{q/2} \|\hat{\beta}_j - \beta_j^*\|_2^q \\
	&\leq (2\bar{\rho})^{1-q} \sum_{j=1}^J T_j^{q/2} \|\hat{\beta}_j - \beta_j^*\|_2^q
\end{aligned}
\tag{B.2}
\label{eq:B.2}
\]

This completes the proof of \eqref{eq:opt12}. It remains to establish \eqref{eq:opt13}. Defining $\Delta_j = \hat{\beta}_j - \beta_j^*$ for $1 \leq j \leq J$, it follows directly from the definition of $\hat{\beta}$ and \Cref{lem:key_inequalities} (ii) that
\[
\begin{aligned}
	\frac{1}{\sqrt{n}} &\left[ \|y - X\hat{\beta}\|_2 - \|y - X\beta^*\|_2 \right] \\
	&\leq \frac{\lambda\alpha}{n} \sum_{j=1}^J T_j^{q/2} \left( \|\beta_j^*\|_2^q - \|\hat{\beta}_j\|_2^q \right) 
	+ \frac{\lambda(1-\alpha)}{n} \sum_{j=1}^J T_j^{q/2} \left( \|D_j\beta_j^*\|_2^q - \|D_j\hat{\beta}_j\|_2^q \right) \\
	&\leq \frac{\lambda\alpha}{n} \sum_{j\in \mathbb{G}}T_j^{q/2} \left[ \|\Delta_j\|_2^q + \frac{1-\alpha}{\alpha}\|D_j\Delta_j\|_2^q \right]  - \frac{\lambda\alpha}{n} \sum_{j\in \mathbb{G}^c} T_j^{q/2} \left[ \|\Delta_j\|_2^q + \frac{1-\alpha}{\alpha}\|D_j\Delta_j\|_2^q \right]
\end{aligned}
\tag{B.3}
\label{eq:B.3}
\]

Since
\[
\left. \frac{\nabla \|y - X\beta\|_2}{\sqrt{n}} \right|_{\beta=\beta^*} = \frac{-X^T\varepsilon}{\sqrt{n}\|\varepsilon\|_2}
\tag{B.4}
\label{eq:B.4}
\]

Then, on $\mathscr{A} \cap \mathscr{B}$, by convexity of the $\ell_2$ norm and \eqref{eq:B.2}, we have
\[
\begin{aligned}
	\frac{1}{\sqrt{n}} \left[ \|y - X\hat{\beta}\|_2 - \|y - X\beta^*\|_2 \right] 
	&\geq -\frac{|\varepsilon^T X \Delta|}{\sqrt{n}\|\varepsilon\|_2} 
	\geq -\sum_{j=1}^J \frac{\|X_j^T \varepsilon\|_2 \|\Delta_j\|_2}{\sqrt{n}\|\varepsilon\|_2}\\
	&\geq -(2\bar{\rho})^{1-q} \frac{V}{n} \sum_{j=1}^J \sqrt{T_j} \|\Delta_j\|_2 
    \geq -\frac{\lambda\alpha}{n\bar{\gamma}} \sum_{j=1}^J T_j^{q/2} \|\Delta_j\|_2^q
\end{aligned}
\tag{B.5}
\label{eq:B.5}
\]

Combining \eqref{eq:B.3} and \eqref{eq:B.5}, we find
\[
\begin{aligned}
	-\frac{\lambda\alpha}{n\bar{\gamma}} \sum_{j=1}^J T_j^{q/2} \|\Delta_j\|_2^q 
	&\leq \frac{\lambda\alpha}{n} \sum_{j\in \mathbb{G}} T_j^{q/2} \|\Delta_j\|_2^q 
	+ \frac{\lambda(1-\alpha)}{n} \sum_{j\in \mathbb{G}} T_j^{q/2} \|D_j\Delta_j\|_2^q \\
	&\quad - \frac{\lambda\alpha}{n} \sum_{j\in \mathbb{G}^c} T_j^{q/2} \|\Delta_j\|_2^q 
	- \frac{\lambda(1-\alpha)}{n} \sum_{j\in \mathbb{G}^c} T_j^{q/2} \|D_j\Delta_j\|_2^q
\end{aligned}
\]

and thus
\[
\begin{aligned}
	\alpha \left(1 - \frac{1}{\bar{\gamma}}\right) \sum_{j\in \mathbb{G}^c} T_j^{q/2} \|\Delta_j\|_2^q 
	+ (1-\alpha) \sum_{j\in \mathbb{G}^c} T_j^{q/2} \|D_j\Delta_j\|_2^q \\
	\leq \alpha \left(1 + \frac{1}{\bar{\gamma}}\right) \sum_{j\in \mathbb{G}} T_j^{q/2} \|\Delta_j\|_2^q 
	+ (1-\alpha) \sum_{j\in \mathbb{G}} T_j^{q/2} \|D_j\Delta_j\|_2^q
\end{aligned}
\]
Noting that \(\bar{\gamma} = \frac{\gamma+1}{\gamma-1}\) and \(\gamma > 1\), we directly obtain the desired result:
\begin{equation}
	\alpha \sum_{j\in \mathbb{G}^c} T_j^{q/2} \|\Delta_j\|_2^q + (1-\alpha) \sum_{j\in \mathbb{G}^c} T_j^{q/2} \|D_j \Delta_j\|_2^q
	\leq \gamma \left[ \alpha \sum_{j\in \mathbb{G}} T_j^{q/2} \|\Delta_j\|_2^q + (1-\alpha) \sum_{j\in \mathbb{G}} T_j^{q/2} \|D_j \Delta_j\|_2^q \right],
	\tag{B.6}
	\label{eq:B.6}
\end{equation}
which completes the proof of this proposition.

\noindent\textbf{\emph{Proof of \Cref{thm:oracle_inequality}}}\quad It follows from \Cref{lem:key_inequalities} (i) and $(2,q)$-GRCC that
\[
\sum_{j\in \mathbb{G}} T_j^{q/2} \|\Delta_j\|_2^q \leq s^{1-q} \left( \sum_{j\in \mathbb{G}} \sqrt{T_j} \|\Delta_j\|_2 \right)^q \leq s^{1-q} \left( \frac{\sqrt{s^*}\|X\Delta\|_2}{\sqrt{n}\kappa} \right)^q,
\tag{B.7}
\label{eq:B.7}
\]
Since $\|D_j\|_2 \leq 2$, one has by \eqref{eq:B.3} and \eqref{eq:B.7} that
\[
\begin{aligned}
	\frac{1}{\sqrt{n}} \left[ \|y - X\hat{\beta}\|_2 - \|y - X\beta^*\|_2 \right] 
	&\leq \frac{\lambda\alpha}{n} \sum_{j\in \mathbb{G}} T_j^{q/2} \|\Delta_j\|_2^q + \frac{\lambda(1-\alpha)}{n} \sum_{j\in \mathbb{G}} T_j^{q/2} \|D_j\Delta_j\|_2^q \\
	&\leq \frac{\lambda\alpha s^{1-q}}{n} \left( \sum_{j\in \mathbb{G}} \sqrt{T_j} \|\Delta_j\|_2 \right)^q + \frac{\lambda(1-\alpha)2^q s^{1-q}}{n} \left( \sum_{j\in \mathbb{G}} \sqrt{T_j} \|\Delta_j\|_2 \right)^q \\
	&\leq \frac{2\lambda s^{1-q}}{n} \left( \frac{\sqrt{s^*}\|X\Delta\|_2}{\sqrt{n}\kappa} \right)^q
\end{aligned}
\tag{B.8}
\label{eq:B.8}
\]

Besides
\[
\frac{\|y - X\hat{\beta}\|_2^2}{n} - \frac{\|y - X\beta^*\|_2^2}{n} = \frac{\|X\Delta - \varepsilon\|_2^2}{n} - \frac{\|\varepsilon\|_2^2}{n} = \frac{\|X\Delta\|_2^2}{n} - \frac{2\varepsilon^T X\Delta}{n}
\tag{B.9}
\label{eq:B.9}
\]
Note that $\bar{\gamma} = \frac{\gamma+1}{\gamma-1}$ and $\gamma > 1$, it follows from \eqref{eq:B.5}, \eqref{eq:B.6} and \eqref{eq:B.7} that
\[
\begin{aligned}
	\frac{\varepsilon^T X\Delta}{n} &\leq \frac{\lambda\alpha}{n\bar{\gamma}} \frac{\|\varepsilon\|_2}{\sqrt{n}} \sum_{j=1}^J T_j^{q/2} \|\Delta_j\|_2^q \\
	&\leq \frac{\lambda}{n\bar{\gamma}} \frac{\|\varepsilon\|_2}{\sqrt{n}} \left[ \alpha(1+\gamma) \sum_{j\in \mathbb{G}} T_j^{q/2} \|\Delta_j\|_2^q + \gamma(1-\alpha) \sum_{j\in \mathbb{G}} T_j^{q/2} \|D_j\Delta_j\|_2^q \right] \\
	&\leq \frac{\lambda\|\varepsilon\|_2}{n^{3/2}} \left( \frac{\alpha(1+\gamma)}{\bar{\gamma}} + \frac{2\gamma(1-\alpha)}{\bar{\gamma}} \right) \sum_{j\in \mathbb{G}} T_j^{q/2} \|\Delta_j\|_2^q \\
	&\leq 2\lambda\gamma \frac{\|\varepsilon\|_2}{n^{3/2}} s^{1-q} \left( \frac{\sqrt{s^*}\|X\Delta\|_2}{\sqrt{n}\kappa} \right)^q
\end{aligned}
\tag{B.10}
\label{eq:B.10}
\]
Combining \eqref{eq:B.8}, \eqref{eq:B.9} and \eqref{eq:B.10}, we get
\[
\begin{aligned}
	\frac{\|X\Delta\|_2^2}{n} &= \frac{\|y - X\hat{\beta}\|_2^2}{n} - \frac{\|y - X\beta^*\|_2^2}{n} + \frac{2\varepsilon^T X\Delta}{n} \\
	&= \left( \frac{\|y - X\hat{\beta}\|_2}{\sqrt{n}} - \frac{\|y - X\beta^*\|_2}{\sqrt{n}} \right) \left( \frac{\|y - X\hat{\beta}\|_2}{\sqrt{n}} + \frac{\|y - X\beta^*\|_2}{\sqrt{n}} \right) + \frac{2\varepsilon^T X\Delta}{n} \\
	&\leq \frac{2\lambda s^{1-q}}{n} \left( \frac{\sqrt{s^*}\|X\Delta\|_2}{\sqrt{n}\kappa} \right)^q \left[ 2\frac{\|\varepsilon\|_2}{\sqrt{n}} + \frac{2\lambda s^{1-q}}{n} \left( \frac{\sqrt{s^*}\|X\Delta\|_2}{\sqrt{n}\kappa} \right)^q \right] 
	\quad + 4\lambda\gamma \frac{\|\varepsilon\|_2 s^{1-q}}{n^{3/2}}  \left( \frac{\sqrt{s^*}\|X\Delta\|_2}{\sqrt{n}\kappa} \right)^q \\
	&\leq 4(\gamma + 1) \frac{\|\varepsilon\|_2}{\sqrt{n}} \frac{\lambda s^{1-q} s^{*q/2}}{n\kappa^q} \left( \frac{\|X\Delta\|_2}{\sqrt{n}} \right)^q + 4\lambda^2 \frac{s^{2(1-q)} s^{*q}}{n^2\kappa^{2q}} \left( \frac{\|X\Delta\|_2}{\sqrt{n}} \right)^{2q}
\end{aligned}
\]
By \Cref{eventA} we have
\[
\frac{\|\varepsilon\|_2}{\sqrt{n}} \leq \sigma\sqrt{1+t}
\tag{B.11}
\label{eq:B.11}
\]
Then it is easy to obtain that
\[
\frac{\|X\Delta\|_2}{\sqrt{n}} \lesssim \max\left\{ \left( \sigma \frac{\lambda s^{1-q} s^{*q/2}}{n\kappa^q} \right)^{\frac{1}{2-q}}, \left( \frac{\lambda s^{1-q} s^{*q/2}}{n\kappa^q} \right)^{\frac{1}{1-q}} \right\}
\]
Invoking condition \eqref{eq:opt11}, the first statement of the theorem follows:
\[
\frac{\|X\Delta\|_2}{\sqrt{n}} \lesssim \left( \sigma \frac{\lambda s^{1-q} s^{*q/2}}{n\kappa^q} \right)^{\frac{1}{2-q}}
\tag{B.12}
\label{eq:B.12}
\]

For the second claim, it follows from \Cref{pro1} and \eqref{eq:B.7} that
\[
\begin{aligned}
	\sum_{j=1}^J T_j^{q/2} \|\Delta_j\|_2^q &\leq (\gamma + 1) \left[ \sum_{j\in \mathbb{G}} T_j^{q/2} \|\Delta_j\|_2^q + \frac{(1-\alpha)}{\alpha} \sum_{j\in \mathbb{G}} T_j^{q/2} \|D_j\Delta_j\|_2^q \right] \\
	&\lesssim s^{1-q} \left( \frac{\sqrt{s^*}\|X\Delta\|_2}{\sqrt{n}\kappa} \right)^q 
	\lesssim \frac{s^{1-q} s^{*q/2}}{\kappa^q} \cdot \left( \sigma \frac{\lambda s^{1-q} s^{*q/2}}{n\kappa^q} \right)^{\frac{q}{2-q}} \\
	&\lesssim \left( \sigma \frac{\lambda s^{2(1/q - 1)} s^*}{n\kappa^2} \right)^{\frac{q}{2-q}}
\end{aligned}
\tag{B.13}
\label{eq:B.13}
\]
which concludes the proof of this theorem.

\noindent\textbf{\emph{Proof of \Cref{pro2}}}\quad
	Decompose the objective function as
	\[
	F(\beta) = f(\beta) + g(\beta) + h(\beta),
	\tag{B.14}
	\label{equation:B.14}
	\]
	where
	\[
	f(\beta) = \|y - X\beta\|_2, \quad g(\beta) = \lambda_1 \sum_{j=1}^J T_j^{q/2} \|\beta_j\|_2^q, \quad h(\beta) = \lambda_2 \sum_{j=1}^J T_j^{q/2} \|D_j \beta_j\|_2^q.
	\]
	
	Since \(y \neq X\beta\), the function \(f\) is continuously differentiable at \(\beta\), with gradient
	\[
	\nabla f(\beta) = -\frac{X^\top (y - X\beta)}{\|y - X\beta\|_2}.
	\]
	Accordingly, the \(j\)-th component of \(\nabla f(\beta)\) is
	\[
	\bigl(\nabla f(\beta)\bigr)_j = -\frac{x_j^\top (y - X\beta)}{\|y - X\beta\|_2}.
	\tag{B.15}
	\label{equation:B.15}
	\]
	
	For each group \(j\), define \(\phi_j(\beta) = \|\beta_j\|_2^q\). By \cite[Lemma~3.1]{Xue2020}, the subdifferential of \(\phi_j\) is
	\[
	\partial \phi_j(\beta) =
	\begin{cases}
		q \|\beta_j\|_2^{q-2} \beta_j, & \beta_j \neq 0, \\
		(-\infty, +\infty)^{T_j}, & \beta_j = 0.
	\end{cases}
	\tag{B.16}
	\label{equation:B.16}
	\]
	Thus, the contribution of \(g\) to the \(j\)-th group is \(\lambda_1 T_j^{q/2} \partial \phi_j(\beta)\).
	
	Next, define \(\psi_j(\beta) = \|D_j \beta_j\|_2^q\). If \(D_j \beta_j \neq 0\), \(\psi_j\) is differentiable by the chain rule, and its gradient is
	\[
	\nabla \psi_j(\beta) = q \|D_j \beta_j\|_2^{q-2} D_j^\top D_j \beta_j.
	\tag{B.17}
	\label{equation:B.17}
	\]
	If \(D_j \beta_j = 0\), the same argument as in \cite[Lemma~3.1]{Xue2020} applied to the composition \(y \mapsto \|y\|_2^q\) with the linear map \(D_j\) yields \(\partial \psi_j(\beta) = (-\infty, +\infty)^{T_j}\).
	
	By \cite[Lemma~3.2]{Xue2020}, \(f\), \(g\), and \(h\) are all regular functions. From \cite[Proposition~10.5]{Rockafellar1998} for separable functions, the sum function is also regular and
	\[
	\partial F(\beta) = \partial f(\beta) + \partial g(\beta) + \partial h(\beta).
	\tag{B.18}
	\label{equation:B.18}
	\]
Consequently, for a fixed group \(j\), if \(\beta_j \neq 0\) and \(D_j \beta_j \neq 0\), each of \(\partial f(\beta)\), \(\partial g(\beta)\), and \(\partial h(\beta)\) is a singleton. Their sum is the vector given in~\eqref{eq:opt18}. Otherwise, at least one of the subdifferentials \(\partial \phi_j(\beta)\) or \(\partial \psi_j(\beta)\) equals \(\mathbb{R}^{T_j}\), so the sum \(\bigl(\partial F(\beta)\bigr)_j\) is also \(\mathbb{R}^{T_j}\). This completes the proof.
	
\noindent\textbf{\emph{Proof of \Cref{lem:boundedness_sufficient_descent}}}\quad (i) Let \(k \geq 1\). By \Cref{lem:group_support_converge}, we have \(\mathbb{G}^{k+1} \subseteq \mathbb{G}^k\) and \(\mathbb{A}^{k+1} \subseteq \mathbb{A}^k\). When \(\mathbb{G}^{k+1} = \mathbb{G}^k\), we have \(\beta^{k+1} = \tilde{\beta}^{k+1}\), and by the first-order upper bound condition, we have
\[
\norm{y - X\beta}_2 \leq \frac{1}{2r_k}\norm{y - X\beta}_2^2 + \frac{r_k}{2},
\]
where \(r_k = \norm{y - X\beta^k}_2\). Then we can introduce a majorization function
\begin{equation}
	M(\beta \mid \beta^k) = \frac{1}{2r_k}\norm{y - X\beta}_2^2 + \frac{r_k}{2} + \lambda_1 \sum_{j\in\mathbb{G}^k} T_j^{q/2} \norm{\beta_j}_2^q + \lambda_2 \sum_{j\in\mathbb{A}^k} T_j^{q/2} \norm{D_j \beta_j}_2^q.
	\tag{B.19}
	\label{eq:B.19}
\end{equation}
Then we have \(F(\beta^{k+1}) \leq M(\beta^{k+1} \mid \beta^k)\).

From the constrained optimization problem \eqref{eq:opt21}, it follows that for \(k \geq 1\), \(\beta_j^k = 0\) for \(j \notin \mathbb{G}^k\), \(D_j\beta_j^k \neq 0\) for \(j \in \mathbb{A}^{k-1}\), and thus
\begin{equation}
	\begin{aligned}
		F(\beta^k) - M(\beta^{k+1} \mid \beta^k) &= \frac{r_k}{2} - \frac{1}{2r_k}\norm{y - X\beta^{k+1}}_2^2 + \lambda_1 \sum_{j\in\mathbb{G}^k} T_j^{q/2} \left[ \norm{\beta_j^k}_2^q - \norm{\beta_j^{k+1}}_2^q \right] \\
		&\quad + \lambda_2 \sum_{j\in\mathbb{A}^{k-1}} T_j^{q/2} \norm{D_j\beta_j^k}_2^q - \lambda_2 \sum_{j\in\mathbb{A}^k} T_j^{q/2} \norm{D_j\beta_j^{k+1}}_2^q \\
		&\geq \frac{r_k}{2} - \frac{1}{2r_k}\norm{y - X\beta^{k+1}}_2^2 + \lambda_1 \sum_{j\in\mathbb{G}^k} T_j^{q/2} \left[ \norm{\beta_j^k}_2^q - \norm{\beta_j^{k+1}}_2^q \right] \\
		&\quad + \lambda_2 \sum_{j\in\mathbb{A}^{k-1} \setminus \mathbb{A}^k} T_j^{q/2} \norm{D_j\beta_j^k}_2^q + \lambda_2 \sum_{j\in\mathbb{A}^k} T_j^{q/2} \left[ \norm{D_j\beta_j^k}_2^q - \norm{D_j\beta_j^{k+1}}_2^q \right].
	\end{aligned}
	\tag{B.20}
	\label{eq:B.20}
\end{equation}
Noting that
\[
\norm{y - X\beta^{k+1}}_2^2 = \norm{y - X\beta^k}_2^2 + \norm{X(\beta^{k+1} - \beta^k)}_2^2 - 2(y - X\beta^k)^\top X(\beta^{k+1} - \beta^k),
\]
we have
\begin{equation}
	\frac{r_k}{2} - \frac{1}{2r_k}\norm{y - X\beta^{k+1}}_2^2 = -\frac{1}{2r_k}\norm{X(\beta^{k+1} - \beta^k)}_2^2 + \frac{1}{r_k}(y - X\beta^k)^\top X(\beta^{k+1} - \beta^k).
	\tag{B.21}
	\label{eq:B.21}
\end{equation}

Since \(D_j\beta_j^{k+1} = 0\) for all \(j \notin \mathbb{G}^k \setminus \mathbb{A}^k\), by the KKT conditions of the constrained optimization problem \eqref{eq:opt21}, we have
\[
\frac{1}{r_k} X_j^\top (X\beta^{k+1} - y) + \lambda_1 T_j^{q/2} w_{1,j}^k \beta_j^{k+1} + \lambda_2 T_j^{q/2} w_{2,j}^k D_j^\top D_j \beta_j^{k+1} = 0,
\]
for all \(j \in \mathbb{A}^k\), and
\[
\frac{1}{r_k} X_j^\top (X\beta^{k+1} - y) + \lambda_1 T_j^{q/2} w_{1,j}^k \beta_j^{k+1} + D_j^\top \mu_j = 0,
\]
for \(j \in \mathbb{G}^k \setminus \mathbb{A}^k\), where \(w_{1,j}^k = q\norm{\beta_j^k}_2^{q-2}\), \(w_{2,j}^k = q\norm{D_j\beta_j^k}_2^{q-2}\), and \(\mu_j\) is the Lagrange multiplier for the equality constraint of the \(j\)-th group. From \Cref{alg:ds-irls-th}, it follows that \(D_j\beta_j^k = 0\) for all \(j \notin \mathbb{A}^{k-1}\) and \(D_j\beta_j^{k+1} = 0\) for all \(j \notin \mathbb{A}^k\). Taking the inner product with \((\beta_j^{k+1} - \beta_j^k)\) yields
\[
\begin{cases}
	\forall j \in \mathbb{A}^k, & 0 = \frac{1}{r_k} (\beta_j^{k+1} - \beta_j^k)^\top X_j^\top (X\beta^{k+1} - y) + \lambda_1 T_j^{q/2} w_{1,j}^k (\beta_j^{k+1} - \beta_j^k)^\top \beta_j^{k+1} \\
	& \quad + \lambda_2 T_j^{q/2} w_{2,j}^k (\beta_j^{k+1} - \beta_j^k)^\top D_j^\top D_j \beta_j^{k+1}, \\
	\forall j \in \mathbb{G}^k \setminus \mathbb{A}^{k-1}, & 0 = \frac{1}{r_k} (\beta_j^{k+1} - \beta_j^k)^\top X_j^\top (X\beta^{k+1} - y) + \lambda_1 T_j^{q/2} w_{1,j}^k (\beta_j^{k+1} - \beta_j^k)^\top \beta_j^{k+1}, \\
	\forall j \in \mathbb{A}^{k-1} \setminus \mathbb{A}^k, & 0 = \frac{1}{r_k} (\beta_j^{k+1} - \beta_j^k)^\top X_j^\top (X\beta^{k+1} - y) \\
	& \quad + \lambda_1 T_j^{q/2} w_{1,j}^k (\beta_j^{k+1} - \beta_j^k)^\top \beta_j^{k+1} - \mu_j^\top D_j \beta_j^k.
\end{cases}
\]
Set \(\mathcal{S}^k = \bigcup_{j\in\mathbb{G}^k} \mathcal{G}_j\). Since \(\beta_j^{k+1} = \beta_j^k = 0\) for all \(j \notin \mathbb{G}^k\), we have \(X_{\mathcal{S}^k}(\beta_{\mathcal{S}^k}^{k+1} - \beta_{\mathcal{S}^k}^k) = X(\beta^{k+1} - \beta^k)\). Then summing the above equalities over all \(j\), we obtain
\begin{equation}
	\begin{aligned}
		0 &= \frac{1}{r_k} (X\beta^{k+1} - y)^\top X(\beta^{k+1} - \beta^k) + \lambda_1 q \sum_{j\in\mathbb{G}^k} T_j^{q/2} \norm{\beta_j^k}_2^{q-2} (\beta_j^{k+1})^\top (\beta_j^{k+1} - \beta_j^k) \\
		&\quad + \lambda_2 q \sum_{j\in\mathbb{A}^k} T_j^{q/2} \norm{D_j\beta_j^k}_2^{q-2} (D_j\beta_j^{k+1})^\top D_j (\beta_j^{k+1} - \beta_j^k) - \sum_{j\in\mathbb{A}^{k-1} \setminus \mathbb{A}^k} \mu_j^\top D_j \beta_j^k.
	\end{aligned}
	\tag{B.22}
	\label{eq:B.22}
\end{equation}
From \eqref{eq:B.22}, we have
\begin{equation}
	\begin{aligned}
		&\frac{1}{r_k} (y - X\beta^k)^\top X(\beta^{k+1} - \beta^k) \\
		&= \frac{1}{r_k} (y - X\beta^{k+1})^\top X(\beta^{k+1} - \beta^k) + \frac{1}{r_k} \norm{X(\beta^{k+1} - \beta^k)}_2^2 \\
		&= \frac{1}{r_k} \norm{X(\beta^{k+1} - \beta^k)}_2^2 + \lambda_1 q \sum_{j\in\mathbb{G}^k} T_j^{q/2} \norm{\beta_j^k}_2^{q-2} (\beta_j^{k+1})^\top (\beta_j^{k+1} - \beta_j^k) \\
		&\quad + \lambda_2 q \sum_{j\in\mathbb{A}^k} T_j^{q/2} \norm{D_j\beta_j^k}_2^{q-2} (D_j\beta_j^{k+1})^\top D_j (\beta_j^{k+1} - \beta_j^k) - \sum_{j\in\mathbb{A}^{k-1} \setminus \mathbb{A}^k} \mu_j^\top D_j \beta_j^k.
	\end{aligned}
	\tag{B.23}
	\label{eq:B.23}
\end{equation}
It follows from \Cref{lem:strong_convexity_ineq} that
\[
\norm{\beta_j^k}_2^q - \norm{\beta_j^{k+1}}_2^q - q\norm{\beta_j^k}_2^{q-2} (\beta_j^k - \beta_j^{k+1})^\top \beta_j^{k+1} \geq \frac{q}{2} \norm{\beta_j^k}_2^{q-2} \norm{\beta_j^{k+1} - \beta_j^k}_2^2,
\]
for \(j \in \mathbb{G}^k\), and
\[
\norm{D_j\beta_j^k}_2^q - \norm{D_j\beta_j^{k+1}}_2^q - q\norm{D_j\beta_j^k}_2^{q-2} (D_j\beta_j^k - D_j\beta_j^{k+1})^\top D_j\beta_j^{k+1} \geq \frac{q}{2} \norm{D_j\beta_j^k}_2^{q-2} \norm{D_j(\beta_j^{k+1} - \beta_j^k)}_2^2,
\]
for \(j \in \mathbb{A}^k\). Then we have
\begin{equation}
	\begin{aligned}
		\sum_{j\in\mathbb{G}^k} T_j^{q/2} \left[ \norm{\beta_j^k}_2^q - \norm{\beta_j^{k+1}}_2^q \right]
		&\geq q \sum_{j\in\mathbb{G}^k} T_j^{q/2} \norm{\beta_j^k}_2^{q-2} (\beta_j^k - \beta_j^{k+1})^\top \beta_j^{k+1} \\
		&\quad + \frac{q}{2} \sum_{j\in\mathbb{G}^k} T_j^{q/2} \norm{\beta_j^k}_2^{q-2} \norm{\beta_j^{k+1} - \beta_j^k}_2^2,
	\end{aligned}
	\tag{B.24}
	\label{eq:B.24}
\end{equation}
\begin{equation}
	\begin{aligned}
		\sum_{j\in\mathbb{A}^k} T_j^{q/2} \left[ \norm{D_j\beta_j^k}_2^q - \norm{D_j\beta_j^{k+1}}_2^q \right]
		&\geq q \sum_{j\in\mathbb{A}^k} T_j^{q/2} \norm{D_j\beta_j^k}_2^{q-2} (D_j\beta_j^k - D_j\beta_j^{k+1})^\top D_j\beta_j^{k+1} \\
		&\quad + \frac{q}{2} \sum_{j\in\mathbb{A}^k} T_j^{q/2} \norm{D_j\beta_j^k}_2^{q-2} \norm{D_j(\beta_j^{k+1} - \beta_j^k)}_2^2.
	\end{aligned}
	\tag{B.25}
	\label{eq:B.25}
\end{equation}
Substituting \eqref{eq:B.21}, \eqref{eq:B.23}, \eqref{eq:B.24}, and \eqref{eq:B.25} into \eqref{eq:B.20} yields
\begin{equation}
	\begin{aligned}
		F(\beta^k) - M(\beta^{k+1} \mid \beta^k) &= -\frac{1}{2r_k}\norm{X(\beta^{k+1} - \beta^k)}_2^2 + \frac{1}{r_k}\norm{X(\beta^{k+1} - \beta^k)}_2^2 \\
		&\quad + \frac{\lambda_1 q}{2} \sum_{j\in\mathbb{G}^k} T_j^{q/2} \norm{\beta_j^k}_2^{q-2} \norm{\beta_j^{k+1} - \beta_j^k}_2^2 \\
		&\quad + \frac{\lambda_2 q}{2} \sum_{j\in\mathbb{A}^k} T_j^{q/2} \norm{D_j\beta_j^k}_2^{q-2} \norm{D_j(\beta_j^{k+1} - \beta_j^k)}_2^2 \\
		&\quad + \lambda_2 \sum_{j\in\mathbb{A}^{k-1} \setminus \mathbb{A}^k} T_j^{q/2} \norm{D_j\beta_j^k}_2^q - \sum_{j\in\mathbb{A}^{k-1} \setminus \mathbb{A}^k} \mu_j^\top D_j \beta_j^k.
	\end{aligned}
	\tag{B.26}
	\label{eq:B.26}
\end{equation}
For the constrained optimization problem \eqref{eq:opt21}, two key conditions hold at the KKT point: the Hessian of the Lagrangian is positive definite, and the gradients of the linear equality constraints are linearly independent (i.e., \(\sigma_{\text{min}}(D_j^\top) > 0\) for each group \(j\)). By Chapter 10 in \citep{Boyd2004}, the Lagrange multiplier \(\mu_j\) is unique. Meanwhile, by Theorem 4.3 in \citep{Movahedian2017}, \(\mu_j\) is bounded with the explicit upper bound
\[
\norm{\mu_j}_2 \leq \frac{l_f}{\sigma_{\text{min}}(D_j^\top)},
\]
where \(l_f\) is the local Lipschitz constant of the objective function \(F(\beta)\), and \(\sigma_{\text{min}}(D_j^\top)\) is the minimum singular value of the matrix \(D_j^\top\).

Recall that for \(j \in \mathbb{A}^{k-1} \setminus \mathbb{A}^k\), we have \(0 < \norm{D_j\beta_j^k}_2 < \tau_2\). By choosing \(\tau_2 \leq \tilde{\tau}_2 = \left( \lambda_2 \sigma_{\text{min}}(D_j^\top) T_{\text{min}}^{q/2} / l_f \right)^{\frac{1}{1-q}}\), we have

\begin{equation}
	F(\beta^k) - M(\beta^{k+1} \mid \beta^k) \geq \frac{\lambda_1 q}{2} \sum_{j\in\mathbb{G}^k} T_j^{q/2} \norm{\beta_j^k}_2^{q-2} \norm{\beta_j^{k+1} - \beta_j^k}_2^2.
	\tag{B.27}
	\label{eq:B.27}
\end{equation}
Since \(F(\beta^{k+1}) \leq M(\beta^{k+1} \mid \beta^k)\), we have
\begin{equation}
	F(\beta^k) - F(\beta^{k+1}) \geq F(\beta^k) - M(\beta^{k+1} \mid \beta^k) \geq \frac{\lambda_1 q}{2} \sum_{j\in\mathbb{G}^k} T_j^{q/2} \norm{\beta_j^k}_2^{q-2} \norm{\beta_j^{k+1} - \beta_j^k}_2^2 \geq 0.
	\tag{B.28}
	\label{eq:B.28}
\end{equation}

When \(\mathbb{G}^{k+1} \subsetneq \mathbb{G}^k\), we denote the truncated residual as \(e^{k+1} = \tilde{\beta}^{k+1} - \beta^{k+1}\), where
\[
e_j^{k+1} =
\begin{cases}
	\tilde{\beta}_j^{k+1}, & j \in \mathbb{G}^k \setminus \mathbb{G}^{k+1}, \\
	0, & j \notin \mathbb{G}^k \setminus \mathbb{G}^{k+1}.
\end{cases}
\]
Since
\begin{equation}
	F(\beta^k) - F(\beta^{k+1}) = \left[ F(\beta^k) - F(\tilde{\beta}^{k+1}) \right] - \left[ F(\beta^{k+1}) - F(\tilde{\beta}^{k+1}) \right].
	\tag{B.29}
	\label{eq:B.29}
\end{equation}
Similar to \eqref{eq:B.28}, one has
\begin{equation}
	F(\beta^k) - F(\tilde{\beta}^{k+1}) \geq \frac{\lambda_1 q}{2} \sum_{j\in\mathbb{G}^k} T_j^{q/2} \norm{\beta_j^k}_2^{q-2} \norm{\tilde{\beta}_j^{k+1} - \beta_j^k}_2^2 \geq 0.
	\tag{B.30}
	\label{eq:B.30}
\end{equation}
Meanwhile, it is easy to obtain that
\[
\begin{aligned}
	F(\beta^{k+1}) - F(\tilde{\beta}^{k+1}) \leq &\norm{y - X\beta^{k+1}}_2 - \norm{y - X\tilde{\beta}^{k+1}}_2 \\
	&-\lambda_1 \sum_{j\in\mathbb{G}^k \setminus \mathbb{G}^{k+1}} T_j^{q/2} \norm{\tilde{\beta}_j^{k+1}}_2^q 
     -\lambda_2 \sum_{j\in\mathbb{G}^k \setminus \mathbb{G}^{k+1}} T_j^{q/2} \norm{D_j \tilde{\beta}_j^{k+1}}_2^q.
\end{aligned}
\]
By the triangle inequality, we have
\[
\norm{y - X\beta^{k+1}}_2 - \norm{y - X\tilde{\beta}^{k+1}}_2 \leq \norm{X e^{k+1}}_2 \leq \norm{X}_2 \norm{e^{k+1}}_2 = \norm{X}_2 \sqrt{\sum_{j\in\mathbb{G}^k \setminus \mathbb{G}^{k+1}} \norm{e_j^{k+1}}_2^2}.
\]
Since \(\norm{e_j^{k+1}}_2 = \norm{\tilde{\beta}_j^{k+1}}_2 < \tau_1\) for all \(j \in \mathbb{G}^k \setminus \mathbb{G}^{k+1}\), we get
\begin{equation}
	F(\beta^{k+1}) - F(\tilde{\beta}^{k+1}) \leq \norm{X}_2 \sqrt{\abs{\mathbb{G}^k \setminus \mathbb{G}^{k+1}} \tau_1^2} \leq \abs{\mathbb{G}^k \setminus \mathbb{G}^{k+1}} \norm{X}_2 \tau_1,
	\tag{B.31}
	\label{eq:B.31}
\end{equation}
where \(\abs{\mathbb{G}^k \setminus \mathbb{G}^{k+1}}\) is the cardinality of the set \(\mathbb{G}^k \setminus \mathbb{G}^{k+1}\). Then for \(k \geq 1\), it follows that \(F(\beta^{k+1}) - F(\beta^k) \leq 0\) when \(\mathbb{G}^{k+1} = \mathbb{G}^k\) and  \(F(\beta^{k+1}) - F(\beta^k) \leq \abs{\mathbb{G}^k \setminus \mathbb{G}^{k+1}} \norm{X}_2 \tau_1\) when \(\mathbb{G}^{k+1} \subsetneq \mathbb{G}^k\). This result clearly extends to the initial iteration \(k=0\). Thus we have
\[
F(\beta^{k+1}) - F(\beta^0) \leq \sum_{i=0}^k \abs{\mathbb{G}^i \setminus \mathbb{G}^{i+1}} \norm{X}_2 \tau_1 \leq J \norm{X}_2 \tau_1,
\]
which means \(F(\beta^{k+1})\) is bounded, and for each group \(j\),
\[
\norm{\beta_j^{k+1}}_2^q \leq \sum_{j=1}^J T_j^{q/2} \norm{\beta_j^{k+1}}_2^q \leq \frac{1}{\lambda_1} F(\beta^{k+1}).
\]
Then there exists \(C > \tau_1\) depending on \(\beta^0\) such that \(\norm{\beta_j^k}_2 \leq C\) for all \(1 \leq j \leq J\). Thus, the proof of statement (i) is completed.

(ii) By \Cref{lem:group_support_converge}, for all \(k \geq K\), we have \(\mathbb{G}^k = \mathbb{G}^{k+1} = \mathbb{G}^*\). Let \(T_{\text{min}} = \min\{T_j: 1 \leq j \leq J\}\). It follows from \eqref{eq:B.28} that for \(k \geq K\), \(F(\beta^k)\) is non-increasing with
\[
F(\beta^k) - F(\beta^{k+1}) \geq \frac{\lambda_1 q}{2} C^{q-2} \sum_{j\in\mathbb{G}^*} T_j^{q/2} \norm{\beta_j^{k+1} - \beta_j^k}_2^2 \geq \frac{\lambda_1 q}{2} T_{\text{min}}^{q/2} C^{q-2} \norm{\beta^{k+1} - \beta^k}_2^2.
\]
Let \(a = \frac{\lambda_1 q}{2} T_{\text{min}}^{q/2} C^{q-2}\). Then inequality \eqref{eq:opt23} holds for all \(k \geq K\).

We next prove that \eqref{eq:opt23} holds for all \(k \geq 0\) under the assumption 
\[
\tau_1 \leq \tilde{\tau}_1 = \min\bigg\{ \bigg( \frac{T_{\text{min}}^{q/2} \lambda_1}{2\|X\|_2} \bigg)^{\frac{1}{1-q}}, \bigg( \frac{\lambda_1}{2\tilde{a}} \bigg)^{\frac{1}{2-q}} \bigg\},
\]
where \(\tilde{a} > 0\) is an arbitrary positive constant. When \(\mathbb{G}^{k+1} \subsetneq \mathbb{G}^k\), for all \(j \in \mathbb{G}^k \setminus \mathbb{G}^{k+1}\), we have
\[
\|X\|_2 \norm{e_j^{k+1}}_2^{1-q} + \tilde{a} T_j^{q/2} \norm{e_j^{k+1}}_2^{2-q} - \lambda_1 T_j^{q/2} \leq 0.
\]
Rearranging gives
\[
\|X\|_2 \norm{e_j^{k+1}}_2 - \lambda_1 T_j^{q/2} \norm{e_j^{k+1}}_2^q \leq -\tilde{a} T_j^{q/2} \norm{e_j^{k+1}}_2^2.
\]
We thus obtain
\begin{equation}
	F(\beta^{k+1}) - F(\tilde{\beta}^{k+1}) \leq -\tilde{a} \sum_{j\in\mathbb{G}^k \setminus \mathbb{G}^{k+1}} T_j^{q/2} \norm{e_j^{k+1}}_2^2.
	\tag{B.32}
	\label{eq:B.32}
\end{equation}
Note that \(\beta_j^{k+1} = \tilde{\beta}_j^{k+1}\) for all \(j \in \mathbb{G}^{k+1}\). For \(j \in \mathbb{G}^k \setminus \mathbb{G}^{k+1}\), we have \(\beta_j^{k+1} = 0\) and \(
2\norm{e_j^{k+1} - \beta_j^k}_2^2 + 2\norm{e_j^{k+1}}_2^2 \geq \norm{\beta_j^k}_2^2
\). Combining \eqref{eq:B.29}, \eqref{eq:B.30}, and \eqref{eq:B.32}, we obtain
\[
\begin{aligned}
	F(\beta^k) - F(\beta^{k+1}) &\geq \frac{\lambda_1 q}{2} \sum_{j\in\mathbb{G}^k} T_j^{q/2} \norm{\beta_j^k}_2^{q-2} \norm{\tilde{\beta}_j^{k+1} - \beta_j^k}_2^2 + \tilde{a} \sum_{j\in\mathbb{G}^k \setminus \mathbb{G}^{k+1}} T_j^{q/2} \norm{e_j^{k+1}}_2^2 \\
	&\geq \frac{\lambda_1 q}{2} \sum_{j\in\mathbb{G}^{k+1}} T_j^{q/2} \norm{\beta_j^k}_2^{q-2} \norm{\beta_j^{k+1} - \beta_j^k}_2^2 \\
	&\quad + \sum_{j\in\mathbb{G}^k \setminus \mathbb{G}^{k+1}} T_j^{q/2} \left[ \frac{\lambda_1 q}{2} \norm{\beta_j^k}_2^{q-2} \norm{e_j^{k+1} - \beta_j^k}_2^2 + \tilde{a} \norm{e_j^{k+1}}_2^2 \right] \\
	&\geq \sum_{j\in\mathbb{G}^{k+1}} \frac{\lambda_1 q}{2} T_{\text{min}}^{q/2} C^{q-2} \norm{\beta_j^{k+1} - \beta_j^k}_2^2 \\
	&\quad + \sum_{j\in\mathbb{G}^k \setminus \mathbb{G}^{k+1}} \min\left\{ \frac{\lambda_1 q}{4} T_{\text{min}}^{q/2} C^{q-2}, \frac{\tilde{a}}{2} T_{\text{min}}^{q/2} \right\} \norm{\beta_j^k - 0}_2^2 \\
	&\geq \sum_{j\in\mathbb{G}^k} \min\left\{ \frac{\lambda_1 q}{4} T_{\text{min}}^{q/2} C^{q-2}, \frac{\tilde{a}}{2} T_{\text{min}}^{q/2} \right\} \norm{\beta_j^k - \beta_j^{k+1}}_2^2.
\end{aligned}
\]

Let \(a = \min\left\{ \frac{\lambda_1 q}{4} T_{\text{min}}^{q/2} C^{q-2}, \frac{\tilde{a}}{2} T_{\text{min}}^{q/2} \right\}\). If \(\tau_1 \leq \tilde{\tau}\), \(F(\beta^k)\) is non-increasing for all \(k \geq 0\), with
\[
F(\beta^k) - F(\beta^{k+1}) \geq a \norm{\beta^{k+1} - \beta^k}_2^2.
\]

This completes the proof of statement (ii).

\noindent\textbf{\emph{Proof of \Cref{lem:subdifferential_boundedness}}}\quad For each \(j = 1, \dots, J\), we construct a groupwise subgradient \(u^{k+1} \in \partial F(\beta^{k+1})\) following \Cref{pro2}, with
\[
u_j^{k+1} = 
\begin{cases}
	-\displaystyle\frac{X_j^\top (y - X\beta^{k+1})}{\norm{y - X\beta^{k+1}}_2} + \lambda_1 q T_j^{q/2} \norm{\beta_j^{k+1}}_2^{q-2} \beta_j^{k+1} + \lambda_2 q T_j^{q/2} \norm{D_j\beta_j^{k+1}}_2^{q-2} D_j^\top D_j \beta_j^{k+1}, & j \in \mathbb{A}^{k+1}, \\
	0, & j \notin \mathbb{A}^{k+1}.
\end{cases}
\]

Let \(r_k = \norm{y - X\beta^k}_2\) and \(r_{k+1} = \norm{y - X\beta^{k+1}}_2\). For any \(j \in \mathbb{A}^{k+1}\), the first-order optimality condition of subproblem \eqref{eq:opt21} gives
\[
\frac{1}{r_k} X_j^\top (X\tilde{\beta}^{k+1} - y) + \lambda_1 T_j^{q/2} q \norm{\beta_j^k}_2^{q-2} \tilde{\beta}_j^{k+1} + \lambda_2 T_j^{q/2} q \norm{D_j\beta_j^k}_2^{q-2} D_j^\top D_j \tilde{\beta}_j^{k+1} = 0.
\]
For \(j \in \mathbb{A}^{k+1}\), we have \(\beta_j^{k+1} = \tilde{\beta}_j^{k+1}\). Recall that \(e^{k+1} = \tilde{\beta}^{k+1} - \beta^{k+1}\), then the above equality can be rewritten as
\[
\frac{1}{r_k} X_j^\top (X\beta^{k+1} + X e^{k+1} - y) + \lambda_1 T_j^{q/2} q \norm{\beta_j^k}_2^{q-2} \beta_j^{k+1} + \lambda_2 T_j^{q/2} q \norm{D_j\beta_j^k}_2^{q-2} D_j^\top D_j \beta_j^{k+1} = 0.
\]
Substituting it into \(u_j^{k+1}\) yields
\[
\begin{aligned}
	u_j^{k+1} &= -\frac{X_j^\top (y - X\beta^{k+1})}{r_{k+1}} + \lambda_1 T_j^{q/2} q \norm{\beta_j^{k+1}}_2^{q-2} \beta_j^{k+1} + \lambda_2 T_j^{q/2} q \norm{D_j\beta_j^{k+1}}_2^{q-2} D_j^\top D_j \beta_j^{k+1} \\
	&= \left( \frac{1}{r_k} - \frac{1}{r_{k+1}} \right) X_j^\top (y - X\beta^{k+1}) - \frac{1}{r_k} X_j^\top X e^{k+1} + \lambda_1 T_j^{q/2} q \left( \norm{\beta_j^{k+1}}_2^{q-2} - \norm{\beta_j^k}_2^{q-2} \right) \beta_j^{k+1} \\
	&\quad + \lambda_2 T_j^{q/2} q \left( \norm{D_j\beta_j^{k+1}}_2^{q-2} - \norm{D_j\beta_j^k}_2^{q-2} \right) D_j^\top D_j \beta_j^{k+1}.
\end{aligned}
\]
Then, we denote \(d^{k+1}, \tilde{u}^{k+1}, \tilde{v}^{k+1} \in \mathbb{R}^p\) as
\[
d_j^{k+1} = 
\begin{cases}
	\displaystyle\left( \frac{1}{r_k} - \frac{1}{r_{k+1}} \right) X_j^\top (y - X\beta^{k+1}) - \frac{1}{r_k} X_j^\top X e^{k+1}, & j \in \mathbb{A}^{k+1}, \\
	0, & j \notin \mathbb{A}^{k+1},
\end{cases}
\]
\[
\tilde{u}_j^{k+1} = 
\begin{cases}
	\lambda_1 T_j^{q/2} q \left( \norm{\beta_j^{k+1}}_2^{q-2} - \norm{\beta_j^k}_2^{q-2} \right) \beta_j^{k+1}, & j \in \mathbb{A}^{k+1}, \\
	0, & j \notin \mathbb{A}^{k+1},
\end{cases}
\]
\[
\tilde{v}_j^{k+1} = 
\begin{cases}
	\lambda_2 T_j^{q/2} q \left( \norm{D_j\beta_j^{k+1}}_2^{q-2} - \norm{D_j\beta_j^k}_2^{q-2} \right) D_j^\top D_j \beta_j^{k+1}, & j \in \mathbb{A}^{k+1}, \\
	0, & j \notin \mathbb{A}^{k+1}.
\end{cases}
\]
It follows that
\begin{equation}
	u^{k+1} = d^{k+1} + \tilde{u}^{k+1} + \tilde{v}^{k+1}.
	\tag{B.33}
	\label{eq:B.33}
\end{equation}

We now bound each component separately. First, by the triangle inequality, we have
\[
|r_{k+1} - r_k| = \left| \norm{y - X\beta^{k+1}}_2 - \norm{y - X\beta^k}_2 \right| \leq \norm{X(\beta^{k+1} - \beta^k)}_2 \leq \norm{X}_2 \norm{\beta^{k+1} - \beta^k}_2.
\]
Since \(\|y - X\beta^k\|_2\) is bounded below, there exists \(\varsigma > 0\) such that \(\|y - X\beta^k\|_2 > \varsigma\) for \(k \geq 0\). Then, we have
\[
\left| \frac{1}{r_k} - \frac{1}{r_{k+1}} \right| = \frac{|r_{k+1} - r_k|}{r_k r_{k+1}} \leq \frac{\norm{X}_2 \norm{\beta^{k+1} - \beta^k}_2}{\varsigma^2}.
\]
Recall that \(e_j^{k+1} = 0\) for all \(j \notin \mathbb{G}^k \setminus \mathbb{G}^{k+1}\). For \(j \in \mathbb{G}^k \setminus \mathbb{G}^{k+1}\), we have \(\beta_j^{k+1} = 0\), \(\norm{e_j^{k+1}}_2 < \tau_1\), and \(\norm{\beta_j^k}_2 \geq \tau_1\). Let \(\mathcal{S}^k = \bigcup_{j\in\mathbb{G}^k} \mathcal{G}_j\), we have
\[
\norm{e^{k+1}}_2 \leq \norm{e^{k+1}_{\mathcal{S}^k \setminus \mathcal{S}^{k+1}}}_2 \leq \norm{\beta^k_{\mathcal{S}^k \setminus \mathcal{S}^{k+1}}}_2 = \norm{\beta^{k+1}_{\mathcal{S}^k \setminus \mathcal{S}^{k+1}} - \beta^k_{\mathcal{S}^k \setminus \mathcal{S}^{k+1}}}_2.
\]
It follows from \Cref{lem:boundedness_sufficient_descent} that \(\|y - X\beta^k\|_2\) is uniformly bounded above. Let \(\tilde{\varsigma} = \sup_k \|y - X\beta^k\|_2 < \infty\), then by the triangle inequality we have
\[
\norm{d^{k+1}}_2 \leq \frac{\norm{X}_2^2 \tilde{\varsigma}}{\varsigma^2} \norm{\beta^{k+1} - \beta^k}_2 + \frac{\norm{X}_2^2}{\varsigma} \norm{\beta^{k+1}_{\mathcal{S}^k \setminus \mathcal{S}^{k+1}} - \beta^k_{\mathcal{S}^k \setminus \mathcal{S}^{k+1}}}_2.
\]

For \(j \in \mathbb{A}^{k+1}\), it follows from \Cref{lem:boundedness_sufficient_descent} that \(\tau_1 \leq \norm{\beta_j^k}_2, \norm{\beta_j^{k+1}}_2 \leq C\). Since \(\phi(t) = t^{q-2}\) is Lipschitz-continuous on \([\tau_1, C]\) with Lipschitz constant \(l_{\tau_1}\), we obtain
\[
\left| \norm{\beta_j^{k+1}}_2^{q-2} - \norm{\beta_j^k}_2^{q-2} \right| \leq l_{\tau_1} \left| \norm{\beta_j^{k+1}}_2 - \norm{\beta_j^k}_2 \right| \leq l_{\tau_1} \norm{\beta_j^{k+1} - \beta_j^k}_2.
\]
Let \(T_{\max} = \max_{1\leq j\leq J} T_j\). Since \(\norm{\beta_j^{k+1}}_2 \leq C\), we have
\[
\norm{\tilde{u}_j^{k+1}}_2 \leq \lambda_1 q T_{\max}^{q/2} l_{\tau_1} C \norm{\beta_j^{k+1} - \beta_j^k}_2.
\]
Thus,
\begin{equation}
	\begin{aligned}
		\norm{\tilde{u}^{k+1}}_2 &\leq \lambda_1 q T_{\max}^{q/2} l_{\tau_1} C \sqrt{\sum_{j\in\mathbb{A}^{k+1}} \norm{\beta_j^{k+1} - \beta_j^k}_2^2} 
		\leq \lambda_1 q T_{\max}^{q/2} l_{\tau_1} C \sqrt{\sum_{j\in\mathbb{G}^{k+1}} \norm{\beta_j^{k+1} - \beta_j^k}_2^2} \\
		&\leq \lambda_1 q T_{\max}^{q/2} l_{\tau_1} C \norm{\beta_{\mathcal{S}^{k+1}}^{k+1} - \beta_{\mathcal{S}^{k+1}}^k}_2.
	\end{aligned}
	\tag{B.35}
	\label{eq:B.35}
\end{equation}

Similarly, since \(D_j\) is the first-order difference matrix, we have \(\|D_j\|_2 \leq 2\) and \(\psi(t) = t^{q-2}\) is Lipschitz continuous on \([\tau_2, 2C]\) with constant \(l_{\tau_2,D}\). For any \(j \in \mathbb{A}^{k+1}\), we obtain
\[
\left| \norm{D_j\beta_j^{k+1}}_2^{q-2} - \norm{D_j\beta_j^k}_2^{q-2} \right| \leq l_{\tau_2,D} \left| \norm{D_j\beta_j^{k+1}}_2 - \norm{D_j\beta_j^k}_2 \right| \leq 2 l_{\tau_2,D} \norm{\beta_j^{k+1} - \beta_j^k}_2.
\]
Note that \(\norm{D_j^\top D_j \beta_j^{k+1}}_2 \leq \norm{D_j}_2^2 \norm{\beta_j^{k+1}}_2 \leq 4C\). We thus obtain
\begin{equation}
	\norm{\tilde{v}^{k+1}}_2 \leq 8 \lambda_2 q T_{\max}^{q/2} l_{\tau_2,D} C \norm{\beta_{\mathcal{S}^{k+1}}^{k+1} - \beta_{\mathcal{S}^{k+1}}^k}_2.
	\tag{B.36}
	\label{eq:B.36}
\end{equation}

Combining \eqref{eq:B.35} and \eqref{eq:B.36}, we get
\[
\norm{\tilde{u}^{k+1}}_2 + \norm{\tilde{v}^{k+1}}_2 \leq \left( \lambda_1 q T_{\max}^{q/2} l_{\tau_1} C + 8 \lambda_2 q T_{\max}^{q/2} l_{\tau_2,D} C \right) \norm{\beta_{\mathcal{S}^{k+1}}^{k+1} - \beta_{\mathcal{S}^{k+1}}^k}_2 \leq \tilde{b} \norm{\beta^{k+1} - \beta^k}_2,
\]
where \(\tilde{b} = \lambda_1 q T_{\max}^{q/2} l_{\tau_1} C + 8 \lambda_2 q T_{\max}^{q/2} l_{\tau_2,D} C\).

Therefore,
\begin{equation}
	\begin{aligned}
		\norm{u^{k+1}}_2 &\leq \norm{d^{k+1}}_2 + \norm{\tilde{u}^{k+1}}_2 + \norm{\tilde{v}^{k+1}}_2 \\
		&\leq \frac{\norm{X}_2^2 \tilde{\varsigma}}{\varsigma^2} \norm{\beta^{k+1} - \beta^k}_2 + \sqrt{2} \max\left\{ \tilde{b}, \frac{\norm{X}_2^2}{\varsigma} \right\} \norm{\beta^{k+1} - \beta^k}_2 
		\leq b \norm{\beta^{k+1} - \beta^k}_2,
	\end{aligned}
	\tag{B.37}
	\label{eq:B.37}
\end{equation}
where
\[
b = \frac{\norm{X}_2^2 \tilde{\varsigma}}{\varsigma^2} + \sqrt{2} \max\left\{ \tilde{b}, \frac{\norm{X}_2^2}{\varsigma} \right\}.
\]
This completes the proof.

\noindent\textbf{\emph{Proof of \Cref{thm:global_convergence}}}\quad From \Cref{lem:boundedness_sufficient_descent} and~\ref{lem:subdifferential_boundedness}, we observe that the function value sequence \(\{F(\beta^k)\}\) has sufficient descent and the subdifferential sequence \(\{\partial F(\beta^k)\}\) is relatively bounded. \Cref{lem:boundedness_sufficient_descent}~(i) further implies that the sequence \(\{\beta^k\}\) is bounded, hence there exists a subsequence \(\{\beta^{k_i}\}\) converging to a point \(\beta^*\). Moreover, the objective function \(F(\beta)\) is semi-algebraic, and thus satisfies the Kurdyka--Łojasiewicz (KL) property on its domain. With all conditions of Theorem~2.9 in~\citep{Attouch2013} satisfied, we conclude that the sequence \(\{\beta^k\}\) converges to a stationary point \(\beta^*\) of \(F(\beta)\), and the finite-length property holds. Thus, statements (i) and (ii) are proven.


\end{spacing}
\end{document}